\documentclass[11pt]{article}

\usepackage{mleftright}      
\mleftright

\usepackage[margin=1.2in]{geometry}
\usepackage{subfigure}
\usepackage{amsfonts}
\usepackage{amsmath,amsthm}
\usepackage{mathrsfs,amssymb}
\usepackage{enumerate}
\usepackage{algorithm}
\usepackage[noend]{algpseudocode}
\usepackage{algorithmicx}
\usepackage[noend]{algpseudocode}
\usepackage{booktabs}
\usepackage{hyperref}
\usepackage{bigints}
\hypersetup{
    colorlinks,
    citecolor=black,
    filecolor=black,
    linkcolor=black,
    urlcolor=black
}
\usepackage[stable]{footmisc}

\DeclareMathOperator*{\argmin}{arg\,min}

\usepackage{tikz}
\usetikzlibrary{arrows,positioning,fit,backgrounds}
\usetikzlibrary{calc}

\usepackage{authblk}

\usetikzlibrary{arrows,positioning,fit,backgrounds,shapes}
\usetikzlibrary{calc}

\newcommand\numberthis{\addtocounter{equation}{1}\tag{\theequation}}

\newtheorem{definition}{Definition}
\newtheorem{remark}{Remark}
\newtheorem{theorem}{Theorem}

\newtheorem{lemma}{Lemma}

\newtheorem{proposition}{Proposition}
\newtheorem{corollary}{Corollary}
\newtheorem{example}{Example}

\newcommand{\be}{\begin{equation}}
\newcommand{\ee}{\end{equation}}
\newcommand{\ben}{\begin{enumerate}}
\newcommand{\een}{\end{enumerate}}
\newcommand{\bea}{\begin{eqnarray}}
\newcommand{\eea}{\end{eqnarray}}
\newcommand{\bean}{\begin{eqnarray*}}
\newcommand{\eean}{\end{eqnarray*}}

\newcommand{\E}{\mathbb{E}}

\newcommand{\R}{\mathbb{R}}

\newcommand{\X}{\mathcal{X}}

\newcommand{\W}[0]{\mathcal{W}}

\newcommand{\A}[0]{\mathcal{A}}

\newcommand{\w}[0]{\mathbf{w}}

\newcommand{\dd}[0]{\mathrm{d}}
\newcommand{\bl}[0]{\bar{\lambda}}

\newcommand{\Y}[0]{\mathcal{Y}}

\newcommand{\s}[0]{\mathsf{s}}

\providecommand{\keywords}[1]
{
  \small	
  \textbf{\textit{Keywords---}} #1
}

\begin{document}

\title{\bf An Entropy-Based Model for Hierarchical Learning}

\author{Amir R. Asadi\footnote{\href{mailto:aa2345@cam.ac.uk}{aa2345@cam.ac.uk}}}
\affil{Statistical Laboratory,\protect \\ University of Cambridge}
\date{January 24, 2023}

\maketitle

\begin{abstract}
Machine learning is the dominant approach to artificial intelligence, through which computers learn from data and experience. In the framework of supervised learning, a necessity for a computer to learn from data accurately and efficiently is to be provided with auxiliary information about the data distribution and target function through the learning model. This notion of auxiliary information relates to the concept of regularization in statistical learning theory. A common feature among real-world datasets is that data domains are multiscale and target functions are well-behaved and smooth. This paper proposes an entropy-based learning model that exploits this data structure and discusses its statistical and computational benefits. The hierarchical learning model is inspired by human beings' logical and progressive easy-to-hard learning mechanism and has interpretable levels. The model apportions computational resources according to the complexity of data instances and target functions. This property can have multiple benefits, including higher inference speed and computational savings in training a model for many users or when training is interrupted. We provide a statistical analysis of the learning mechanism using multiscale entropies and show that it can yield significantly stronger guarantees than uniform convergence bounds. 

\end{abstract}

\section{Introduction}
\subsection{Background}
Nowadays, machine learning is the dominant approach to artificial intelligence, through which computers learn from data and experience. In the framework of supervised learning, data examples are assumed to emerge randomly as pairs $(X, Y)$, where $X$ and $Y$ are called the data instance and the data label, respectively. There exists a target function $Y=f(X)$ which maps any data instance to its corresponding label. 
A computer is given a sequence of data instances with their corresponding labels, independently drawn from the underlying data probability distribution. Then, it is expected to learn the target function $f$ relating data instances with their corresponding labels and to predict the label of new randomly drawn data instances. An important observation is that the sequence of training examples usually does not contain all the missing information about the target function. Separate from the training data, \emph{auxiliary information} should also be given through the learning model so that the computer can learn the target function accurately and efficiently. To shed more light on this discussion, consider the extreme case when no such auxiliary information is given to a computer. In other words, the computer has no information about the target function or the data domain besides the training sequence. What is the optimal task that it can do in this case? The best job is to memorize the training examples, which most likely results in \emph{overfitting} and poor performance on new data instances. Similarly, much of every human being's knowledge about any task is learned based on training data plus prior knowledge and intuition.  
This notion of auxiliary information relates with the concept of {``regularization''} studied in statistical learning theory, see the paper \cite{vapnik1999nature}. This concept has different forms, including restricting the hypothesis class and explicitly and implicitly regularizing the training mechanism. It is also connected to the no-free-lunch theorem {(see, for example, \cite[Theorem 5.1]{SSS}),} which implies that every learning algorithm should have some prior knowledge about the underlying data probability distribution and the target function to succeed.
 
A common feature among real-world datasets is that {data domains} are multiscale.  That is, data emerge in different scales of magnitude and have a variety of sizes and complexities. This fact has been used in different topics, for example, wavelet theory, Fourier analysis, and signal processing (see, for example, the book \cite{weinan2011principles}). Many examples of empirical data distributions in physics, biology, and social sciences are from multiscale distributions 
  see, for example, the paper \cite{10.2307/25662336}. Moreover, {target functions} in the real world are usually well-behaved and can be written as compositions of simple functions.  

How can one exploit such information about real-world data and target functions 
in a machine learning model to gain both statistically and computationally?
To that aim, in this paper, we propose a model for learning from multiscale data and smooth target functions using a compositional learning architecture and a hierarchical training mechanism. Our learning model is inspired by the logical learning mechanisms of human beings, which learn different tasks progressively from easy to difficult examples.
Multiscale data domains and smooth target functions appear ubiquitously in the real world, examples of which are medical datasets, financial datasets, biological datasets, natural language processing, etc. Thus, our learning model 
may have many applications. 
\subsection{Overview of This Paper}
In our proposed learning model, we exploit the multiscale structure of data domains and the smoothness of their target functions. {Throughout the paper,} for the sake of simplicity, we assume that data instances and data labels are one-dimensional real numbers. We elaborate more with the following two assumptions:
\begin{enumerate}[(a)]
	\item Data instances $X\in \R$ emerge in different scales of magnitude. For example, this can be modeled by assuming that $X$ has a {power-law distribution $\mu$}, see Section \ref{sec: bounding chained risk example}. We assume that the domain consists of $d$ separate scales: For a sequence of scale parameters $\gamma_0<\gamma_1<\dots <\gamma_d$, let the domain set $\X=\X_1\cup \dots \cup \X_d$ be partitioned into sets $\X_k:=\{\gamma_{k-1}\leq |X|<\gamma_k\}$, where $\X_k$ denote the domain of data at scale $i$ based on the norm of $X$. 
	\item Target functions $Y=f(X)\in \R$ are well-behaved and smooth. {In this paper,} we model this by assuming that the target function $f$ is invertible, and $f$ and its inverse $f^{-1}$ are twice-differentiable functions. {Functions with these properties are known in the literature as diffeomorphism functions (see, for example, \cite{hirsch2012differential})}. 
\end{enumerate}
 
The proposed hierarchical model learns by starting from the easy smaller-scaled examples and progressing towards the more difficult larger-scaled ones. To elaborate more, {we consider the following definition:}
The \emph{dilation} of the function $f$ at scale $\gamma$ is defined as
\begin{equation*}
	f_{[\gamma]}(x):=\frac{f(\gamma x)}{\gamma},
\end{equation*} 
which can be interpreted as a ``zoomed'' version of the original function, where the amount of zooming in is determined by $\gamma$. {The following important observation appears in the proof of \cite[Theorem 2]{bartlett2018representing} and is based on the smoothness property of $f$: The functions $f_{[\gamma_{k}]}$, $1\leq k \leq d$, interpolate between $f=f_{[\gamma_d]}$ and $f_{[\gamma_0]}$, and $f_{[\gamma_0]}$ is very close to a linear function (the derivative of $f$ at the origin) if $\gamma_{0}$ is small. }
We define the following multiscale decomposition of $f$:
\begin{equation*}
	f(x)=\Delta_d \circ \dots \circ \Delta_1 \circ f_{[\gamma_0]}(x),
\end{equation*}
 where $\Delta_k(x):=f_{[\gamma_k]}\circ f^{-1}_{[\gamma_{k-1}]}$ for all $1\leq k \leq d$, and call it the \emph{ladder decomposition} of $f$. Here we used the notation for function composition: $(g_1\circ g_2)(x)=g_1(g_2(x))$. 
Due to the smoothness property of $f$, if subsequent scale parameters $\gamma_k$ and $\gamma_{k-1}$ are close to each other, then we expect that $\Delta_k$, which is a transformation between the dilation of $f$ at scale $\gamma_{k-1}$ with the finer dilation $f_{[\gamma_k]}$, to be close to the identity function. Thus, we expect function $\Delta_k$ to be \emph{simple}: its difference with the identity function does not vary too much and is easy to predict and learn. Consider a $d$-level hierarchical learning model defined as follows:
\begin{equation*}
	h_k(x):=h_{k-1}(x) + \mathcal{F}(h_{k-1}(x),w_k), \textrm{ for all } 1\leq k\leq d.
\end{equation*}
 We train this model so that each $h_k$ approximates the dilation $f_{[\gamma_k]}$, with the following hierarchical procedure: First, observe data examples at the smallest scale $\X_1$, corresponding to the easiest examples, and learn $\Delta_1$ by sampling $w_1$ from a Gibbs measure (a maximum entropy distribution). Then, given the learned $\Delta_1$, observe data belonging to the next scale of magnitude $\X_2$, learn $\Delta_2$ similarly by sampling $w_2$ from a Gibbs measure, and repeat this process. Therefore, learning the target function for data at smaller scales can be described as ``easier'' and is a stepping stone for learning more difficult cases at higher scales. This will then be a mathematical model for easy-to-hard learning where scale plays the role of time, similar to how a human learns a topic (or a course) from working on the easiest examples and progressing to the hardest examples. Hierarchically learning the ladder decomposition is analogous to step-by-step climbing a ladder.
 We elaborate more precisely on the model in the next sections.

Our proposed learning model has the following merits:
\begin{enumerate}
	\item {\bf Hierarchical Learning Model with Interpretable Levels:} In the proposed learning model, training is performed to make each level $h_k$ approximate the dilation of $f$ at scale $\gamma_{k}$, that is, $f_{[\gamma_{k}]}$. In other words, the mapping between the input and any level of the learning model is made close to a scaled-down version of the target function. Thus, the role of each level of our compositional learning model has an interpretation, in contrast to black-box hierarchical models. 
	Moreover, in our learning model, the complexity of a particular data instance is modeled with its \emph{norm}. To extend this model to other applications, the type of norm may be chosen appropriately for the given problem: for example, in finance, the higher the amount of revenue resources of a company (larger norm), the more difficult it is to predict fraud. In medical imaging and image processing, the higher the amount of sparsity (lower $\ell_0$-norm), the easier it is to predict its target. In natural language processing, the norm can correspond to the frequency level. Similarly, the complexity of a target function can be interpreted with its Lipschitz norm: the higher the variability of any function, the more difficult it is to predict its value. 
	\item {\bf Computational Savings in Inference:} An important computational benefit of the learned model is the following: To compute the output of the model given input data instance $x\in \X_k$ and to predict its label (if $x$ is a new data instance), it suffices to process this input only for $k$ levels and compute $h_k({x})$. Namely, one does not need to pass the instance through all of the $d$ levels of the learning model; only the first $k$ levels are enough. This is true because the model is trained such that $h_k({x})$ approximates the dilation $f_{[\gamma_k]}({x})$ with which one can compute the value of $f({x})$ with appropriate rescaling, since $|{x}|<\gamma_k$. 
	In other words, when using the trained model for predicting the target of new data, the amount of computation on computing the output of the learned model given that particular data instance as input is proportionate to the complexity or \emph{difficulty} of that instance. 
		Since, by assumption, data instances are distributed heterogeneously at different scales and difficulties, 
	this fact can result in significant computational savings and higher inference speed.
	\item {\bf Computational Savings in Training for Heterogeneous Users:} The proposed learning model can provide computational savings when there are $d$ different users, each requiring accurate prediction of the labels of data instances at scale $k$, that is, $\X_k$. Instead of training separate models for each of the $d$ users, we streamline the process by using the trained model for user $k-1$ to train the model for user $k$, for all $1\leq k \leq d$.
	\item {\bf Interruption During Training:} Training of current machine learning models on massive datasets may take a very long time. Our training mechanism consists of $d$ stages. Even if this mechanism terminates after stage $k$, for any reason, we can still guarantee a useful model with which one can accurately predict the label of data instances belonging to $\X_1\cup\dots \cup \X_k $.
	\item {\bf Multiscale Statistical Analysis Stronger Than Uniform Convergence:} The statistical analysis of the risk of the trained compositional model is tailored to the hierarchical training mechanism and takes its multiscale structure into account when deriving the bound on its statistical risk. Unsurprisingly, the bound can be much sharper than a uniform convergence bound for the empirical-risk-minimizing hypothesis.
\end{enumerate}

This work has been inspired mainly by combining ideas from \cite{bartlett2018representing} and from \cite{asadi2020chaining}. The paper \cite{bartlett2018representing} shows that any smooth bi-Lipschitz function $f$ can be represented as a composition of $d$ functions $f_d\circ \dots \circ f_1$ where each function $f_k$, $1\leq k \leq d$, is close to the identity function $I$ in the sense that the Lipschitz constant of $f_k-I$ decreases inversely with $d$. Notably, the proof of \cite[Theorem 2]{bartlett2018representing} is based on using the notion of dilations of smooth functions. The paper \cite{asadi2020chaining} obtains the solution to the multiscale entropy regularization problem which can be interpreted as a multiscale extension of Gibbs measures. However, it is not possible to efficiently sample from the probability distribution. {By using its proof technique} and by taking a reverse approach, in this paper, we derive the multiscale-entropic regularized loss function that a self-similar and computable distribution minimizes. Moreover, in contrast to the work of \cite{asadi2020chaining}, the output of the hierarchical learning model is read from different levels depending on the scale of the input data instance. In other words, we match the multiscale architecture of the learning model with the multiscale data domain.
\subsection{Further Related Work} 
Information-theoretic approaches to learning and analysis of the generalization error of learning algorithms have been devised in the context of PAC-Bayesian bounds \cite{mcallester1999pac}, and later in a related form using mutual information in, for example, the work of \cite{Russo,xu2017information, bu2020tightening}. These information-theoretic methods have been extended to multiscale techniques \cite{audibert2007combining,asadi2018chaining, asadi2020chaining, clerico2022chained}. The paper \cite{xu2017information} further analyzes the statistical risk of the Gibbs measure, also called maximum-entropy training, and a multiscale extension of this result has been derived \cite{asadi2020chaining}. In the work \cite{raginsky2016information}, some other information-theoretic measures of the stability of learning algorithms and bounds on their generalization ability have been derived. 

 Multiscale entropies, a linear combination of entropies of a system at different scales, implicitly appear in the classical chaining technique of probability theory. For example, one can rewrite Dudley inequality (\cite{Dudley}) variationally and transform the bound into a linear mixture of metric entropies at multiple scales. These multiscale measures have been further studied \cite{asadi2020maximum}.

In the paper \cite{fletcher2012decomposing}, it is shown that any diffeomorphism defined on the sphere can be written as the composition of bi-Lipschitz functions with small distortion.

Hierarchical learning models composed of near-identity layers have been used in the context of neural network learning as residual networks \cite{he2016deep} and via a dynamical systems approach \cite{weinan2017proposal}. 

\subsection{Organization of This Paper}
{The rest of the paper is organized as follows: }We first provide the preliminaries and notation in Section \ref{sec: prelim}. Then, in Section \ref{sec: ladder decom}, we present the definition of ladder decompositions for diffeomorphisms and study the Lipschitz continuity and smoothness of each rung of this decomposition. In Section \ref{sec: model}, our proposed learning model is described. Section \ref{sec: entropic analysis} consists of two parts: In Subsection \ref{sec: multiscale entropic training chained risk}, we show that the multiscale maximum-entropy type of training achieves low \emph{chained risk}. Then, in Subsection \ref{sec: bounding chained risk example}, we show if the data distribution $\mu$ is a power-law distribution, then the chained risk can bound the statistical risk from above, hence overall yielding an upper bound on the statistical risk of the learned model. Section \ref{sec: bounded norm} exemplifies that a set of Lipschitz functions -- functions with bounded Lipschitz norm -- can be represented with a parameterized model with bounded-norm parameters. Finally, in Section \ref{sec: Conclusions} we discuss the conclusions of our work.
\section{Preliminaries and Notation}\label{sec: prelim}
Throughout {the} paper, $|\cdot |$ indicates Euclidean distance, $I$ denotes the identity function, and $\mu^{\otimes n}$ denotes the $n$ times tensor product of measure $\mu$ with itself. The set of real numbers and integers are denoted with $\R$ and $\mathbb{Z}$, respectively.

Random variables and random vectors are represented with capital letters, while small letters are used for their realizations. Throughout the paper, $U$ denotes the equiprobable (uniform) probability distribution where its support set is indicated with a subscript. For a random variable $X$ and probability measure $P$, the notation $X\sim P$ means that $X$ is distributed according to $P$. We denote the Dirac probability measure on $\mathsf{w}$ as $\delta_{\mathsf{w}}$.

We first state some information-theoretic definitions and tools which we use in our analysis in Section \ref{sec: entropic analysis}. For two distributions $P$ and $Q$, $P\ll Q$ means that $P$ is absolutely continuous with respect to $Q$.
\begin{definition}[Entropy]
	The Shannon entropy of a discrete random variable $X$ taking values on $\mathcal{A}$ is defined as
	\begin{equation*}
		H(X):=-\sum_{a\in \mathcal{A}}P_X(a)\log P_X(a).
	\end{equation*}
	The relative entropy between two distributions $P_X$ and $Q_X$, if $P_X\ll Q_X$ is defined as 
	\begin{equation*}
		D(P_X\|Q_X):=\sum_{a\in \mathcal{A}}P_X(a)\log\left(\frac{P_X(a)}{Q_X(a)} \right),
	\end{equation*} 
	otherwise, we define $D(P_X\|Q_X):=\infty$. 
	The conditional relative entropy is defined as 
\begin{align}
	D\left(P_{Y|X}\middle\|Q_{Y|X}\middle|P_X\right)&=\int D(P_{Y|X=\omega}\|Q_{Y|X=\omega})\mathrm{d}P_X(\omega)\nonumber\\
						   &=\E\left[D\left(P_{Y|X}(\cdot|X)\middle\|Q_{Y|X}(\cdot|X)\right)\right], \quad X\sim P_X. \nonumber
\end{align}
\end{definition}
The following extremely useful property of entropy is called the ``chain rule''. For proof, see, for example, \cite[Theorem 2.5.3]{Cover}:
\begin{lemma}[Entropy Chain Rule]\label{chain rule lemma} Let $P_{XY}$ and $Q_{XY}$ be two distributions. We have
\begin{equation}
	D\left(P_{XY}\|Q_{XY}\right)=D(P_X\|Q_X)+D\left(P_{Y|X}\middle\|Q_{Y|X}\middle|P_X\right). \nonumber
\end{equation}
\end{lemma}
The next definition relates to ``geometric'' transformations of probability measures:
\begin{definition}[Scaled and Tilted Distributions] Given a discrete probability measure $P$ defined on a set $\mathcal{A}$, and any $\lambda\in [0,1]$, we define the scaled distribution $(P)^{\lambda}$ for all $a\in \A$ as 
\begin{equation*}
	(P)^{\lambda}(a):=\frac{(P(a))^{\lambda}}{\sum_{x\in \A}(P(x))^{\lambda}}.
\end{equation*}
Given two discrete probability measures $P$ and $Q$ defined on a set $\A$, and any $\lambda\in [0,1]$, we define the tilted distribution $(P,Q)^{\lambda}$ as the following geometric mixture:
\begin{equation*}
	(P,Q)^{\lambda}(a):= \frac{P^{\lambda}(a)Q^{1-\lambda}(a)}{\sum_{x\in \mathcal{A}}P^{\lambda}(x)Q^{1-\lambda}(x)}.
\end{equation*}
\end{definition}
Clearly, if $U$ is the equiprobable distribution on $\A$, then 
\begin{equation*}
	(P)^{\lambda}=(P,U)^{\lambda}.
\end{equation*}
We require the definition of R\'{e}nyi divergence to properly state the next lemma.
\begin{definition}[R\'{e}nyi Divergence]
For discrete distributions $P$ and $Q$ defined on a set $\mathcal{A}$ and for any $\lambda \in (0, 1) \cup (1,\infty)$, the R\'{e}nyi divergence is defined as
\begin{equation}
    D_{\lambda}(P\|Q):=\frac{1}{\lambda-1}\log\left(\sum_{a\in\mathcal{A}}P^{\lambda}(a)Q^{1-\lambda}(a) \right).\nonumber
\end{equation}
For $\lambda=1$, we define $D_{\lambda}(P\|Q):=D(P\|Q)$. 
\end{definition} 
In our analysis in Section \ref{sec: entropic analysis}, similar to the paper of \cite{asadi2020chaining}, we encounter linear combinations of relative entropies. The next lemma shows the role of tilted distributions in such linear combinations. For a proof, see \cite[Theorem 30]{van2014renyi}:
\begin{lemma}[Entropy Combination]\label{lem: Tilted distribution entropy}
Let $\lambda\in [0,1]$. For any distributions $P,Q$ and $R$ such that $P \ll Q$ and $P\ll R$, we have
\begin{equation}
	\lambda D(P\|Q)+(1-\lambda)D(P\|R)=D\left(P\middle\|(Q,R)^{\lambda}\right)+(1-\lambda)D_{\lambda}(Q\|R).\nonumber
\end{equation}	
\end{lemma}
We provide the following definition to later simplify the notation in the proof of Theorem \ref{thm: self-similar Gibbs minimizer}:
\begin{definition}[Congruent Functionals]\label{def: congruent functionals}
	We call two functionals $\mathcal{L}_1(P)$ and $\mathcal{L}_2(P)$ of a distribution $P$ congruent and write $\mathcal{L}_1\cong \mathcal{L}_2$ if $\mathcal{L}_1-\mathcal{L}_2$ does not depend on $P$.
\end{definition}
For example, Lemma \ref{lem: Tilted distribution entropy} implies that if $Q$ and $R$ are fixed distributions, then as functionals of $P$, the following congruency holds:
\begin{equation}\nonumber\label{eq: congruent relation tilted distribution}
	\lambda D(P\|Q)+(1-\lambda)D(P\|R)\cong D\left(P\middle\|(Q,R)^{\lambda}\right).
\end{equation}
Specifically, if $U$ is the equiprobable distribution, then
\begin{equation}\label{eq: congruent relation tilted distribution}
	\lambda D(P\|Q)+(1-\lambda)D(P\|U)\cong D\left(P\middle\|Q^{\lambda}\right).
\end{equation}

The following well-known result, sometimes referred to as the Gibbs variational principle, implies that the distribution that minimizes the sum of average energy (loss) and entropy (regularization) is a Gibbs measure:  
\begin{lemma}[Gibbs Measure]\label{lem: Gibbs relative entropy}
  Let ${\W}$ be an arbitrary finite set. Given a function $f: {\W}\rightarrow \mathbb{R}$ and $\lambda>0$,  
  we define the following Gibbs probability measure for all $w\in \W$:
  \begin{equation*}
 {Q}_{W}(w)\triangleq\frac{\exp\left(-\frac{f(w)}{\lambda}\right)}{\sum_{v\in {\W}}\exp\left(-\frac{f(v)}{\lambda}\right)}.
  \end{equation*}
 Then, for any probability measure $P_{W}$ defined on $\W$, we have
  \begin{equation}
      \nonumber
      \mathbb{E}[f(W)]+\lambda D(P_W\|U_W)=\lambda D\left(P_{W}\middle\|Q_W\right)-\lambda \log \left(\sum_{v\in {\W}}\exp\left(-\frac{f(v)}{\lambda}\right)\mathrm{d}v\right),
  \end{equation} 
  where $W\sim P_W$.
\end{lemma}
Particularly, Lemma \ref{lem: Gibbs relative entropy} yields the following congruency identity as functionals of $P_W$:
\begin{equation}\label{eq: congruent relation Gibbs}
	\E [f(W)]+\lambda D(P_W\|U_W)\cong \lambda D\left(P_W\middle\|Q_W\right).
\end{equation}
We later make use of the congruency relations \eqref{eq: congruent relation tilted distribution} and \eqref{eq: congruent relation Gibbs} iteratively in the proof of Theorem \ref{thm: self-similar Gibbs minimizer}.

Next, we present some definitions of {regularity} of functions and the related notation. 
\begin{definition}[Lipschitz and Smooth Function]
	Let $\mathcal{V}$ be a compact subset of $\R$. A differentiable function $f:\mathcal{V}\to \R$ is $M_1$-Lipschitz if for all $x,y\in \mathcal{V}$,
	\begin{equation*}
		|f(x)-f(y)|\leq M_1|x-y|.
	\end{equation*}
	We say that the function $f$ is $M_2$-smooth if its derivative $f'$ is $M_2$-Lipschitz.
\end{definition}
Clearly, if $f$ is a $M_1$-Lipschitz function, then we have $|f'(x)|\leq M_1$ for all $x\in \R.$
In this paper, we require both $f$ and its inverse $f^{-1}$ to be well-behaved functions, as described in the following definition:
\begin{definition}[Diffeomorphism] Let $\mathcal{V}$ be a compact subset of $\R$.  A function $f:\mathcal{V}\to \R$ is an $(M_1,M_2)$-diffeomorphism if it is invertible and both $f$ and its inverse $f^{-1}$ are twice differentiable, $M_1$-Lipschitz and $M_2$-smooth.
\end{definition}
If $f$ is a $(M_1,M_2)$-diffeomorphism, we then have $M_1\geq 1$, since 
\begin{equation*}
	\frac{|x-y|}{M_1}\leq |f(x)-f(y)|\leq M_1|x-y|.
\end{equation*}
Next, we define the notion of \emph{dilation} of a function, which is equivalent to a rescaled version of it:
\begin{definition}[Dilation]
	For any $0<\gamma\leq 1$, the dilation of function $f$ at scale $\gamma$ is defined as 
\begin{equation*}
	f_{[\gamma]}(x):= \frac{f(\gamma x)}{\gamma}.
\end{equation*}
\end{definition}

	It is easy to prove that if $f$ is invertible, then so is its dilation $f_{[\gamma]}$, where 
\begin{equation*}
	\left(f_{[\gamma]}\right)^{-1}(x)= \frac{f^{-1}(\gamma x)}{\gamma}=f_{[\gamma]}^{-1}(x).
\end{equation*}
In other words, the inverse of the dilation is identical to the dilation of the inverse function.
{Now, we prove the following proposition:}
\begin{proposition}\label{prop: Lipshitz dilation to linear}
	Let $f:(-R,R)\to \R$ be a $M_2$-smooth function. Then, $f_{[\gamma]}(x)-f'(0)x$ is $\gamma M_2R$-Lipschitz.
\end{proposition}
\begin{proof}
	The proof is based on a simple extension of the proof of \cite[Theorem 2]{bartlett2018representing}. Let $x,y\in (-R, R)$. Based on the mean value theorem, there exists $z$ in between $x$ and $y$ such that $f(\gamma x)-f(\gamma y) = \gamma f'(\gamma z)(x-y).$ We can write
	\begin{align*}
		\left|(f_{[\gamma]}(x)-f'(0)x) - (f_{[\gamma]}(y)-f'(0)y) \right|
		&=\left|\left(\frac{f(\gamma x)}{\gamma}-\frac{f(\gamma y)}{\gamma} \right)-f'(0)(x-y)\right|\\
		&=\left|f'(\gamma z)(x-y)-f'(0)(x-y) \right|\\
		&=\left|x-y\right||f'(\gamma z)-f'(0)|\\
		&\leq |x-y||\gamma z|M_2\\
		&\leq \gamma M_2R|x-y|.
	\end{align*}
\end{proof}
In particular, if $f(0)=0$, then we have 
\begin{equation*}
	\left|(f_{[\gamma]}(x)-f'(0)x) \right|\leq \gamma M_2R|x|.
\end{equation*}
In Section \ref{sec: entropic analysis}, we require the following well-known result on one specific type of Kolmogorov (quasi-arithmetic) mean:
\begin{lemma}[Kolmogorov Mean]\label{Kolmogorov mean lemma}
	Let $\mathsf{z}=(z_1,z_2,\dots,z_N)\in \mathbb{R}^N$. For $\lambda\geq 0$, the following weighted average 
	\begin{equation*}
		G_{\lambda}(\mathsf{z}):= -\lambda\log\left(\frac{1}{N}\sum_{j=1}^N \exp\left(-\frac{z_j}{\lambda}\right) \right),
	\end{equation*}
	which is a type of Kolmogorov mean,
	satisfies
	\begin{equation*}
		\min_{j=1,\dots,n} z_j-\lambda \log N\leq G_{\lambda}(\mathsf{z})\leq \min_{j=1,\dots,n} z_j.
	\end{equation*}
\end{lemma}
The proof of Theorem \ref{thm: chained risk of multiscale entropic training} requires the following tools of the topic of concentration of measures:
\begin{definition}[Subgaussian] A random variable $X$ is called $\sigma$-subgaussian if for all $\lambda\in \R $, its cumulant generating function satisfies
\begin{equation*}
	\log \E\left[e^{\lambda(X-\E X)} \right]\leq \frac{\lambda^2\sigma^2}{2}.
\end{equation*}
\end{definition}
The following result is based on \cite[Lemma 1]{xu2017information}, which itself can be derived from the transportation lemma of \cite[Lemma 4.18]{Lugosi}:
\begin{lemma}\label{lemma: subgaussian difference independence}
	If $g(\bar{A},\bar{B})$ is $\sigma$-subgaussian where $(\bar{A},\bar{B})\sim P_AP_B$, then for all $\lambda>0$,
	\begin{equation*}
		\E \left[g(\bar{A},\bar{B})\right] - \E \left[g(A,B)\right]\leq {\lambda}\left(\log |\mathcal{A}|-H(A|B) \right)+\frac{\sigma^2}{2\lambda}.
	\end{equation*}
\end{lemma}
The Azuma--Hoeffding inequality shows the subgaussianity of the sum of independent and bounded random variables: 
\begin{lemma}[Azuma--Hoeffding]\label{lem: Azuma Hoeffding}
Let $X_1,\dots,X_n$ be independent random variables such that $a\leq X_i\leq b$ for all $i$. 
Then,
\begin{equation*}
	\E \left[e^{\frac{\lambda}{n}\sum_{i=1}^n \left(X_i-\E X_i\right)} \right]\leq \frac{\lambda^2(b-a)^2}{2n}
\end{equation*} 
In other words, $\sum_{i=1}^nX_i/n$ is $(b-a)/\sqrt{n}$-subgaussian.
\end{lemma}
The following well-known lemma, which we use in Section \ref{sec: bounded norm}, bounds the approximation error of the Reimann sum (see, for example, the book \cite{hughes2020calculus}):
\begin{lemma}\label{Reimann sum approximation lemma} The approximation error of the Reimann sum is bounded as follows:
\begin{equation*}
	\left|\int_a^b \hat{f}(x)\dd x - S_{\mathrm{R}}\right|\leq \frac{\hat{M}_1(b-a)^2}{2n},
\end{equation*}
where $\hat{M}$ is the maximum absolute value of the derivative of $f$ on $[a,b]$.
\end{lemma}
In the framework of supervised batch learning, $\mathcal{X}$ represents the instances domain, $\mathcal{Y}$ denotes the labels domain, and $\mathsf{Z}=\mathcal{X}\times \mathcal{Y}$ is the examples domain. $\mathcal{H}=\{h_w : w\in \mathcal{W}\}$ is the hypothesis set, where the hypotheses are indexed by an index set $\mathcal{W}$. Let $\ell:\W\times \mathsf{Z}\to \mathbb{R}^+$ be a loss function. A learning algorithm receives a random training sequence $S=(Z_1,Z_2,...,Z_n)$ of $n$ examples with i.i.d.\ random elements drawn from $\mathsf{Z}$ with an unknown distribution $\mu$. Namely, $S\sim \mu^{\otimes n}$. In the training procedure, it chooses $h_W\in\mathcal{H}$ according to a random transformation $P_{W|S}$. For any $w\in\mathcal{W}$, let 
$
L_{\mu}(w):= \E[\ell(w,Z)] 
$
denote the statistical (or population) risk of hypothesis $h_w$, where $~ Z\sim \mu$. The aim of statistical learning is to choose a learning algorithm for which the expected statistical risk $\E[L_{\mu}(W)]$ is small. 
\section{Ladder Decompositions of Smooth Functions}\label{sec: ladder decom}
In this section, we show that any diffeomorphism defined on a bounded interval $(-R, R)$ can be decomposed at multiple scales into what we name \emph{ladder decomposition}, such that different layers (rungs) of this decomposition are smooth and Lipschitz with small Lipschitz norm.

\begin{definition}[Ladder Decomposition]\label{def: ladder decomposition} Let $d\geq 1$ be an arbitrary integer. 
	Consider a sequence of scale parameters $\{\gamma_k\}_{k=0}^d$ such that $0<\gamma_0<\gamma_1<\dots <\gamma_d =1$. For any function $f:\R\to \R$
and for all $1\leq k \leq d$, let
\begin{equation*}
	\Delta_k := f_{[\gamma_k]}\circ f_{[\gamma_{k-1}]}^{-1},
\end{equation*}
and
\begin{equation*}
	\psi_k(x):=\Delta_k(x)-x.
\end{equation*}
Clearly, for all $1\leq k\leq d$, we have 
\begin{equation*}
	f_{[\gamma_k]}= \Delta_k \circ \dots \circ \Delta_1\circ f_{[\gamma_0]}.
\end{equation*}
In particular,
\begin{equation}\label{ladder decomposition equation}
	f = \Delta_d \circ \Delta_{d-1} \circ \dots \circ \Delta_1\circ f_{[\gamma_0]}.
\end{equation}
We call \eqref{ladder decomposition equation} the ladder decomposition of function $f$ at scale parameters $\{\gamma_k\}_{k=0}^d$.
\end{definition}
Based on Proposition \ref{prop: Lipshitz dilation to linear}, the smaller $\gamma_k$ is, the closer function $f_{[\gamma_k]}(x)$ is to the linear function $f'(0)x$. 
We intuitively expect that if two subsequent scale parameters $\gamma_k$ and $\gamma_{k-1}$ are close, then $\Delta_k$ is a function close to the identity function. The next theorem precisely formulates this intuition and is a key result for the rest of the paper. Assume that $C_1:=3M_1M_2$ and $C_2:=M_2(M_1^2+M_1)$.
\begin{theorem}\label{Ladder Decomposition Theorem}
	Let $f:(-R,R)\to \R$ be a $(M_1,M_2)$-diffeomorphism. 
	For all $1\leq k\leq d$, the function $\psi_k(x)$ is $C_1R(\gamma_k-\gamma_{k-1})$-Lipschitz and $C_2$-smooth.  
\end{theorem}
\begin{proof} 
For any $1\leq k \leq d$, let $x,y$ be arbitrary elements of the domain of $\psi_k$.
We can write
	\begin{align*}
		&|\psi_k(x)-\psi_k(y)|\\
		&=\left|\Delta_k(x)-x -(\Delta_k(y)-y) \right|\\
		&=\left|\frac{f(\gamma_kv)}{\gamma_k}-\frac{f(\gamma_{k-1}v)}{\gamma_{k-1}} -\left(\frac{f(\gamma_ku)}{\gamma_k}-\frac{f(\gamma_{k-1}u)}{\gamma_{k-1}} \right) \right|\\
		&= \left|\frac{1}{\gamma_k}\left(f(\gamma_ku)-f(\gamma_{k-1}u) - (f(\gamma_kv)-f(\gamma_{k-1}v))\right)- \frac{(\gamma_k-\gamma_{k-1})}{\gamma_k\gamma_{k-1}}(f(\gamma_{k-1}u)-f(\gamma_{k-1}v))\right|,
	\end{align*}
	where $v := f_{[\gamma_{k-1}]}^{-1}(x)$ and $u := f_{[\gamma_{k-1}]}^{-1}(y)$. 
	Define $r(z):= f(\gamma_kz)-f(\gamma_{k-1}z)$. We have $r'(z)=\gamma_kf'(\gamma_kz)-\gamma_{k-1}f'(\gamma_{k-1}z)$. Based on the mean value theorem, there exists $z_1,z_2$ between $u$ and $v$ such that
	$
		r(u)-r(v)=r'(z_1)(u-v)
	$ and $f(\gamma_{k-1}u)-f(\gamma_{k-1}v)=f'(\gamma_{k-1}z_2)\gamma_{k-1}(u-v)$. Note that $|z_1|,|z_2|\leq \max\{|u|,|v|\}\leq R$.
	 Hence,
	\begin{align*}
		&\frac{1}{\gamma_k}\left(f(\gamma_ku)-f(\gamma_{k-1}u) - (f(\gamma_kv)-f(\gamma_{k-1}v))\right) - \frac{(\gamma_k-\gamma_{k-1})}{\gamma_k\gamma_{k-1}}(f(\gamma_{k-1}u)-f(\gamma_{k-1}v))\\
		&\quad = \frac{1}{\gamma_k}(r(u)-r(v)) - \frac{(\gamma_k-\gamma_{k-1})}{\gamma_k\gamma_{k-1}}(f(\gamma_{k-1}u)-f(\gamma_{k-1}v))\\
		&\quad = \frac{(u-v)}{\gamma_k}(\gamma_kf'(\gamma_k z_1)-\gamma_{k-1}f'(\gamma_{k-1}z_1)) - \frac{(\gamma_k-\gamma_{k-1})}{\gamma_k\gamma_{k-1}}(u-v)\gamma_{k-1}f'(\gamma_{k-1}z_2)\\
		&\quad = (u-v)\left( f'(\gamma_k z_1)-\frac{\gamma_{k-1}}{\gamma_k}f'(\gamma_{k-1}z_1) - \frac{(\gamma_k-\gamma_{k-1})}{\gamma_k}f'(\gamma_{k-1}z_2)\right)\\
		&\quad = \frac{(u-v)}{\gamma_k}\left( \gamma_{k-1}(f'(\gamma_k z_1)-f'(\gamma_{k-1}z_1))+(\gamma_k-\gamma_{k-1})(f'(\gamma_k z_1)-f'(\gamma_{k-1}z_2))\right).
	\end{align*}
	Therefore,
	\begin{align*}
		\left|\psi_k(x)-\psi_k(y) \right|&=\frac{|u-v|}{\gamma_k}\left| \gamma_{k-1}(f'(\gamma_k z_1)-f'(\gamma_{k-1}z_1))+(\gamma_k-\gamma_{k-1})(f'(\gamma_k z_1)-f'(\gamma_{k-1}z_2))\right|\\
		& \leq \frac{|u-v|}{\gamma_k}\left( \gamma_{k-1}M_2(\gamma_k-\gamma_{k-1})|z_1|+ (\gamma_k-\gamma_{k-1})M_2 \gamma_k(|z_1|+|z_2|)\right)\\
		& \leq \frac{|u-v|}{\gamma_k}\left( \gamma_{k-1}M_2(\gamma_k-\gamma_{k-1})R+ (\gamma_k-\gamma_{k-1})M_2 \gamma_k(2R)\right)\\
		& \leq \frac{|u-v|}{\gamma_k}\left( \gamma_kM_2(\gamma_k-\gamma_{k-1})R+ 2(\gamma_k-\gamma_{k-1})M_2 \gamma_kR\right)\\
		& \leq 3M_2R (\gamma_k-\gamma_{k-1})|u-v|. \numberthis \label{difference u and v 1}
	\end{align*}
	Since $f^{-1}$ is $M_1$-Lipschitz, we have
	\begin{align*}
		|x-y|&= \left|f_{[\gamma_{k-1}]}(u)-f_{[\gamma_{k-1}]}(v)\right|\\
			 &= \frac{1}{\gamma_{k-1}}\left|f(\gamma_{k-1}u) - f(\gamma_{k-1}v) \right|\\
			 & \geq  \frac{1}{M_1\gamma_{k-1}}|\gamma_{k-1}(u-v)|\\
			 & = \frac{|u-v|}{M_1}.\numberthis \label{difference u and v 2}
	\end{align*}
	Combining \eqref{difference u and v 1} and \eqref{difference u and v 2}, we get
	\begin{equation*}
		\left|\psi(x)-\psi(y) \right|\leq C_1R(\gamma_k-\gamma_{k-1})|x-y|.
	\end{equation*}
	Hence, $\psi(x)$ is $C_1R(\gamma_k-\gamma_{k-1})$-Lipschitz.

We now prove the smoothness property. Let $g(x):=f^{-1}(x)$. Based on the chain rule of derivatives, we can write 
	\begin{align*}
		\Delta_k'(x)&=\frac{1}{\gamma_k}f'\left(\frac{\gamma_k}{\gamma_{k-1}}g(\gamma_{k-1}x) \right)\frac{\gamma_k}{\gamma_{k-1}}g'(\gamma_{k-1}x)\gamma_{k-1}\\
		&=f'\left(\frac{\gamma_k}{\gamma_{k-1}}g(\gamma_{k-1}x) \right)g'(\gamma_{k-1}x).
	\end{align*}
	Therefore,
	\begin{align*}
		\Delta_k''(x)&=f''\left(\frac{\gamma_k}{\gamma_{k-1}}g(\gamma_{k-1}x) \right) \frac{\gamma_k	}{\gamma_{k-1}} g'(\gamma_{k-1}x) \gamma_{k-1}g'(\gamma_{k-1}x)\\
		&\quad + f'\left(\frac{\gamma_k}{\gamma_{k-1}}g(\gamma_{k-1}x) \right)g''(\gamma_{k-1}) \gamma_{k-1}\\
		&=\gamma_kf''\left(\frac{\gamma_k}{\gamma_{k-1}}g(\gamma_{k-1}x) \right)\left(g'(\gamma_{k-1}x) \right)^2+\gamma_{k-1}f'\left(\frac{\gamma_k}{\gamma_{k-1}}g(\gamma_{k-1}x) \right)g''(\gamma_{k-1}).
	\end{align*}
	Based on the assumption that $f$ and $g$ are both $M_1$-Lipschitz and $M_2$-smooth and $\gamma_{k-1},\gamma_{k}\leq 1$, we deduce
	\begin{align*}
		|\psi_k''(x)|&=|\Delta_k''(x)|\\
				     &\leq M_2(M_1^2+M_1)\\
				     &=C_2.
	\end{align*}
	Therefore, $\psi_k(x)$ is $C_2$-smooth.
\end{proof}

\begin{remark}\rm
	The proof of \cite[Theorem 2]{bartlett2018representing} only implies that, for all $1\leq k \leq d$, function $\psi_k$ is $C(\gamma_k-\gamma_{k-1})/\gamma_{k}$-Lipschitz for some constant $C$, which is weaker than our result. For example, when scale parameters are chosen as $\gamma_k=k/d$, then our result yields that $\psi_k$ is $O(1/d)$-Lipschitz, whereas \cite[Theorem 2]{bartlett2018representing} only concludes that $\psi_k$ is $O((\log d)/d)$-Lipschitz. However, our result is currently restricted to functions $f$ with domain and range in $\R$.
\end{remark}

We now present an example in which the functions $\psi_k(x)$, $1\leq k\leq d$, have a closed-form expression:
\begin{example}\rm Let $f=\tanh(x)$. We have $f^{-1}(x)=\frac{1}{2}\ln\left(\frac{1+x}{1-x}\right)$. Assume that $\gamma_k=2^{k-d}$ for all $0\leq k \leq d$. We can write
\begin{align*}
	\psi_k(x)&=\Delta_k(x)-x\\
			 &= f_{[\gamma_k]}\circ f_{[\gamma_{k-1}]}^{-1}(x)-x\\
			 &=\frac{\exp\left(2\gamma_kf_{[\gamma_{k-1}]}^{-1}(x)\right)-1}{\gamma_k\left(\exp\left(2\gamma_kf_{[\gamma_{k-1}]}^{-1}(x)\right)+1\right)}-x\\
			 &=\frac{\exp\left(2\gamma_k\frac{1}{2\gamma_{k-1}}\ln\left(\frac{1+\gamma_{k-1}x}{1-\gamma_{k-1}x}\right)\right)-1}{\gamma_k\left(\exp\left(2\gamma_k\frac{1}{2\gamma_{k-1}}\ln\left(\frac{1+\gamma_{k-1}x}{1-\gamma_{k-1}x}\right)\right)+1\right)}-x\\
			 &=\frac{\left(\frac{1+\gamma_{k-1}x}{1-\gamma_{k-1}x}\right)^{\frac{\gamma_k}{\gamma_{k-1}}}-1}{\gamma_k\left(\left(\frac{1+\gamma_{k-1}x}{1-\gamma_{k-1}x}\right)^{\frac{\gamma_k}{\gamma_{k-1}}}+1\right)}-x\\
			 &=\frac{2\gamma_{k-1}x}{\gamma_k\left(1+\gamma_{k-1}^2x^2\right)}-x\\
			 &=\frac{x}{1+\gamma_{k-1}^2x^2}-x\\
			 &=-\frac{\gamma_{k-1}^2x^3}{1+\gamma_{k-1}^2x^2}\\
			 &=-\frac{1}{\gamma_{k-1}^{-2}x^{-3}+x^{-1}}
\end{align*}
Figure \ref{plot2b} depicts the plot of $\psi_k(x)$ for all $1\leq k \leq d$, where $d=5$.

\begin{figure}[H]
	\centering
	\includegraphics[width=.4\linewidth,height=.4\linewidth]{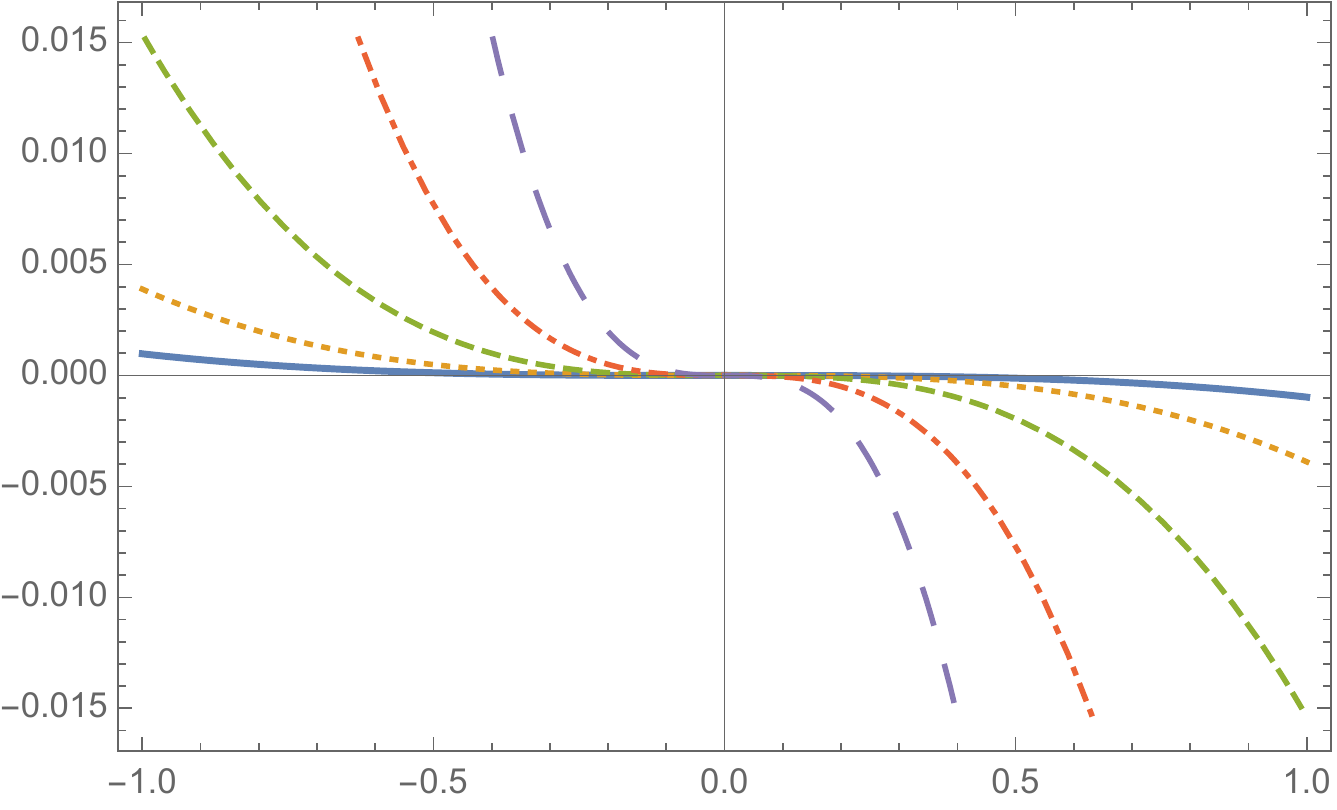}
	\label{plot2b} 
	\caption{\small  A plot of  $\psi_k(x)$ for different values $k$: $k=1$ the solid line, $k=2$ the dotted line, $k=3$,  the dashed line,  $k=4$, the dotted-dashed line and $k=5$ the large dashed line.}
	\label{plot23}
\end{figure}
\end{example}
\section{The Proposed Learning Model}\label{sec: model}
In this section, we precisely formulate the learning model.
Let $d\geq 1$ be an integer and $\varepsilon>0$ and $\beta>1$ be real numbers. Assume that the data instance domain is defined as
\begin{equation*}
	\X:=\{x\in \mathbb{R}: \varepsilon \leq |x|<R\},
\end{equation*}
where $R:=\varepsilon\beta^d$.
Let the scale parameters $\{\gamma_k\}_{k=0}^d$ form a geometric sequence such that for all $0\leq k \leq d$, 
\begin{equation}\label{eq: scale geometric sequence}
	\gamma_k:=\beta^{k-d}.
\end{equation}
For all $1\leq k \leq d$, we define 
\begin{align*}
	\X_k &:= \{x\in \R: \varepsilon\beta^{k-1}\leq |x|< \varepsilon\beta^k\}\\
		 &= \{x\in \R: R\gamma_{k-1}\leq |x|< R\gamma_k\}.
\end{align*}
Clearly, 
$
	\X=\cup_{k=1}^d\X_k.
$
We call each set $\X_k$ the domain of instances at \emph{scale} $k$. It is obvious that $\varepsilon$ and $R$ are the smallest and largest magnitude of data instances, respectively.

We let the label set to be $\Y=\R$ and assume that the target function $f$ is a $(M_1,M_2)$-diffeomorphism. Suppose that from previous knowledge or intuition, we know $f_{[\gamma_0]}$, that is, we know the behavior of function $f$ at extremely small scales $0\leq |x|<\varepsilon$. For example, it may be assumed that $f_{[\gamma_0]}$ is equal to the derivative of $f$ at the origin, namely, the linear function $f'(0)x$.
 For simplicity, we further assume that $f(0)=0$.
 Based on Theorem \ref{Ladder Decomposition Theorem}, For all $1\leq k\leq d$, the function $\psi_k(x)$ is $C_1 \varepsilon(\beta-1)\beta^{k-1}$-Lipschitz and $C_2$-smooth.

To learn the target function $f$ on $\X$ progressively and stage by stage, we define the following $d$-level hiearchical learning model:
Assume that $h_0(x):= f_{[\gamma_0]}(x)$ and for all $1\leq k \leq d$,
\begin{equation}\label{model formulation 1}
	h_{k}(x) := h_{k-1}(x)+ \mathcal{F}\left(h_{k-1}(x),w_k \right).
\end{equation}
We call ${w}_k$ the parameters of the $k$th level of the learning model and allow it to take a value from a set $\W_k$ during training. Let $\w=(w_1,\dots,w_d)$ be the sequence of parameters of this model. The learning model aims to make the mapping between the input and each layer $h_k$ approximate the dilated version of the target function $f$ at scale $\gamma_k$, that is, $f_{[\gamma_k]}$. 
Thus, for a successfully trained model, $\mathcal{F}(h_{k-1}(x),w_k)$ should well approximate $\psi_k(h_{k-1}(x))$. Given that $\psi_k(\cdot)$ is $C_1R(\gamma_k-\gamma_{k-1})$-Lipschitz and by assumption $\psi_k(0)=0$, it is enough to assume that for all $w_k\in \W_k$, 
\begin{equation}\label{eq: bounded norm of levels}
	\left|\mathcal{F}\left(\cdot,w_k\right)\right|\leq \rho_k,
\end{equation}
where $\rho_k\approx C_1R(\gamma_k-\gamma_{k-1})$.
The function $\mathcal{F}$ should be chosen per enough representation power of the model; we give an example in Section \ref{sec: bounded norm}.

Given a successfully trained model, the number of steps that an input instance $x\in \X$ needs to be processed is proportionate to the scale of the instance magnitude $|x|$. This can be interpreted as a measure of the difficulty or complexity of that particular instance. More precisely, we define the output of the model $h(x)$ as follows:
\begin{equation*}
	h_{\mathsf{w}}(x):= 
	\begin{cases}
		\gamma_{1}h_{1}\left(\frac{x}{\gamma_{1}}\right) &\textrm{\quad if \quad } x\in \X_1\\
		
		\gamma_{2}h_{2}\left(\frac{x}{\gamma_{2}}\right) &\textrm{\quad if \quad } x\in \X_{2},\\
		&\vdots
		\\
		h_d(x) &\textrm{\quad if \quad } x\in \X_{d}.\\
	\end{cases} 
\end{equation*}

Let the $n$-tuple of training instances be denoted with $\mathsf{s}=(x_1,\dots, x_n)$. We assume that we are given the instance-label pairs $(x_i,f(x_i))$ for all $1\leq i\leq n$. We now mathematically model the training mechanism. This mechanism starts from the simplest training examples whose instances are at the smallest scale $\X_1$, and progressively trains the layers of the model by using the larger-scaled (more complex) examples. At each level, corresponding to each scale of the data, training is modeled as sampling $w_k$ from a Gibbs measure with loss (energy) as the empirical risk evaluated for that specific scale of the training data. It is well-known that such Gibbs measures are maximum-entropy distributions, see the paper of \cite{jaynes1957information}. 
Precisely, for all $1\leq k \leq d$, given trained values for $w_1^{k-1}$, we sample the vector value for $w_k$ from the following probability distribution: 
\begin{align*}
	P^{*}_{W_k|W_1^{k-1}}\left(w_k\middle|w_1^{k-1}\right)&= \frac{\exp \left( -\frac{1}{n\lambda_k}\sum_{x_i\in \mathsf{s}\cap \X_k} \left|\gamma_kh_k\left(\frac{x_i}{\gamma_k}\right)-f(x_i)\right|\right)}{\sum_{w_k'\in \W_k}\exp \left( -\frac{1}{n\lambda_k}\sum_{x_i\in \mathsf{s}\cap \X_k} \left|\gamma_kh_k'\left(\frac{x_i}{\gamma_k}\right)-f(x_i)\right|\right)},
\end{align*}
where $h'_k$ denote the levels of a learning model with parameters $(w'_1,\dots,w_d')$.
For this reason, this training mechanism is hierarchical, stochastic, and self-similar (at each scale we sample from a Gibbs measure with a similar loss function). We call this training mechanism \emph{multiscale entropic training}. This mechanism has the following benefit: if we stop training after sampling the first $k$ levels, then we are guaranteed to have a useful trained model for data in $\X_1\cup X_2\cup \dots \cup \X_k$.

\section{Analysis of the Learning Model} \label{sec: entropic analysis}
In this section, using multiscale entropies, we statistically analyze the learning model's performance. In Subsection \ref{sec: multiscale entropic training chained risk}, we prove that the multiscale entropic training mechanism achieves low \emph{chained risk}. Then, in Subsection \ref{sec: bounding chained risk example}, we provide an example of the data instance distribution $\mu$, a power-law probability distribution, with which we can bound the statistical risk based on the chained risk. Subsection \ref{sec: bounded norm} shows a parameterization example and analyzes its representation power.
\subsection{Multiscale Entropic Training and Chained Risk}\label{sec: multiscale entropic training chained risk}
Let $\mathsf{w}:=(w_1,\dots,w_d)$. 
For all $1\leq k \leq d$, we define $w_1^k :=  (w_1,\dots,w_k)$ and 
\begin{align}\label{eq: definition of k-risk}
	\ell_k\left(w_1^k,x\right)&:=
	\begin{cases}
		\left|\gamma_kh_k\left(\frac{x}{\gamma_k}\right)-f(x)\right| &\textrm{ if } x\in\X_k \\
		0 &\textrm{ if } x\notin\X_k.
	\end{cases} 
\end{align}
Clearly, the loss of the model on example $(x,f(x))$ is $\ell(\mathsf{w},x)=\sum_{k=1}^d \ell_k\left(w_1^k,x\right).$ 

Recall that the $n$-tuple of training instances is denoted with $\mathsf{s}=(x_1,\dots,x_n)$. For all $1\leq k \leq d$, we define
\begin{equation*}
	\ell_k\left(w_1^k,\mathsf{s}\right):=\frac{1}{n}\sum_{i=1}^n \ell_k\left(w_1^k,x_i\right).
\end{equation*}
Based on the definition of the model in the previous section, for all $1\leq k \leq d$, we have
\begin{equation*}
	P^{*}_{W_k|W_1^{k-1}}\left(w_k\middle|w_1^{k-1}\right)= \frac{\exp\left(-\frac{\ell_k\left(w_1^k,\mathsf{s}\right)}{\lambda_k}\right)}{\sum_{w_k'\in \W_k}\exp\left(-\frac{\ell_k\left(w_1^{k-1}w'_k,\mathsf{s}\right)}{\lambda_k}\right)},
\end{equation*}
	and 
\begin{equation*}
	P^{*}_{\mathsf{W}}=P^{*}_{W_1}P^{*}_{W_2|W_1}\dots P^{*}_{W_d|W_1^{d-1}}.
\end{equation*}
For simplicity in the notation, henceforth we assume that $\lambda_{d+1}:=0$. The next theorem indicates that the self-similar measure $P^{*}_{\mathsf{W}}$ is the minimizing distribution of the sum of a multiscale loss and a multiscale entropy. Let
\begin{equation*}
	{\ell}^{(\lambda)}(\mathsf{w},\mathsf{s}):=\sum_{k=1}^d \left(\ell_k\left(w_1^k,\mathsf{s}\right) - \overline{\ell}_k\left(w_1^{k-1},\s\right) \right),
\end{equation*}
where for all $1\leq k\leq d$,
\begin{align*}
	\overline{\ell}_k\left(w_1^{k-1},\s \right):= -\lambda_k\log \left(\frac{1}{|\W_k|}\sum_{ w'_k\in \W_k}\exp \left(-\frac{\ell_k\left( w_1^{k-1}w'_k,\s \right)}{\lambda_k}\right) \right).
\end{align*}
\begin{theorem}\label{thm: self-similar Gibbs minimizer}
We have 
\begin{equation}
	P^{*}_{\mathsf{W}}=\argmin_{P_{\mathsf{W}}}\left\{\E \left[{\ell}^{(\lambda)}(\mathsf{W},\s )\right] - \sum_{k=1}^d(\lambda_k-\lambda_{k+1})H\left(W_1^k\right) \right\}.\label{eq: the minimization}
\end{equation}
\end{theorem}

\begin{proof} We develop what we call the ``multiscale congruent technique'' used in the proof of \cite[Theorem 13]{asadi2020chaining}. Specifically, recalling the definition of congruent functionals in Definition \ref{def: congruent functionals}, we aim to show that as a functional of $P_{\mathsf{W}}$,
\begin{equation}
	\E\left[{\ell}^{(\lambda)}(\mathsf{W},\mathsf{s})\right]- \sum_{k=1}^{d} (\lambda_k-\lambda_{k+1})H\left(W_1^k \right)\cong \sum_{k=1}^d \lambda_k D\left(P_{W_k|W_1^{k-1}}\middle\|P^{\star}_{W_k|W_1^{k-1}}\middle|P_{W_1^{k-1}}\right). \label{eq: vanishing entropies 2}
\end{equation}
This would then immediately imply \eqref{eq: the minimization}, as setting $P_{\mathsf{W}}=P^{\star}_{\mathsf{W}}$ makes all entropies in the right side of \eqref{eq: vanishing entropies 2} vanish together and . 

For all $1\leq k \leq d$, define
\begin{equation*}
	L_k\left(w_1^k,\s\right):=\sum_{j=1}^k \ell_j\left(w_1^j,\s \right),
\end{equation*}
and 
\begin{equation*}
	Q^{(k)}_{W_1^k}\left(w_1^k\right):=\frac{\exp \left(-\frac{L_k\left(w_1^k,\s\right)}{\lambda_k} \right)}{\sum_{v_1^k\in \W_1\times \dots \times \W_k}\exp \left(-\frac{L_k\left(v_1^k,\s\right)}{\lambda_k} \right)}.
\end{equation*}
We can write
\begin{align}
	\hat{\mathcal{L}}\left(P_{W_1^d} \right)&:=\E \left[L_d\left(W_1^d,\s \right) \right] + \sum_{k=1}^{d} (\lambda_k-\lambda_{k+1})D\left(P_{W_1^k}\middle\|U_{W_1^k}\right)\nonumber \\
	&= \sum_{k=1}^{d-1} (\lambda_k-\lambda_{k+1})D\left(P_{W_1^k}\middle\|U_{W_1^k}\right)+\left(\E \left[L_d\left(W_1^d,\s \right) \right]+\lambda_d D\left(P_{W_1^d}\middle\|U_{W_1^d} \right)\right)\nonumber\\
	&\cong \sum_{k=1}^{d-1} (\lambda_k-\lambda_{k+1})D\left(P_{W_1^k}\middle\|U_{W_1^k}\right)+\lambda_d D\left(P_{W_1^d}\middle\|Q^{(d)}_{W_1^d} \right)\label{eq: cong1}\\
	&=\sum_{k=1}^{d-1} (\lambda_k-\lambda_{k+1})D\left(P_{W_1^k}\middle\|U_{W_1^k}\right)\nonumber\\
	&\quad +\lambda_d D\left(P_{W_1^{d-1}}\middle\|Q^{(d)}_{W_1^{d-1}}\right) +\lambda_d D\left(P_{W_d|W_1^{d-1}}\middle\|Q^{(d)}_{W_d|W_1^{d-1}}\middle|P_{W_1^{d-1}}\right)\label{eq: chain rule}\\
	&=\sum_{k=1}^{d-2} (\lambda_k-\lambda_{k+1})D\left(P_{W_1^k}\middle\|U_{W_1^k}\right)\nonumber\\
	&\quad +\left((\lambda_{d-1}-\lambda_{d})D\left(P_{W_1^{d-1}}\middle\|U_{W_1^{d-1}}\right)+\lambda_d D\left(P_{W_1^{d-1}}\middle\|Q^{(d)}_{W_1^{d-1}}\right)\right)\nonumber\\
	&\quad +\lambda_d D\left(P_{W_d|W_1^{d-1}}\middle\|Q^{(d)}_{W_d|W_1^{d-1}}\middle|P_{W_1^{d-1}}\right)\nonumber\\
	&\cong \sum_{k=1}^{d-2} (\lambda_k-\lambda_{k+1})D\left(P_{W_1^k}\middle\|U_{W_1^k}\right)\nonumber +\lambda_{d-1}D\left(P_{W_1^{d-1}}\middle\|\left(Q^{(d)}_{W_1^{d-1}}\right)^{\frac{\lambda_d}{\lambda_{d-1}}}\right)\nonumber\\
	&\quad +\lambda_d D\left(P_{W_d|W_1^{d-1}}\middle\|Q^{(d)}_{W_d|W_1^{d-1}}\middle|P_{W_1^{d-1}}\right)\label{eq: cong2},
\end{align}
where \eqref{eq: cong1} is based on Lemma \ref{lem: Gibbs relative entropy}, \eqref{eq: chain rule} is based on the chain rule of entropy (Lemma \ref{chain rule lemma}), and \eqref{eq: cong2} is based on Lemma \ref{lem: Tilted distribution entropy}. 
Note that
  \begin{align*}
  	Q^{(d)}_{W_1^{d-1}}\propto \exp\left(-\frac{ L_{d-1}\left(w_1^{d-1},\mathsf{s}\right)}{\lambda_d}\right)\left(\sum_{w_d\in \W_d}\exp\left(-\frac{\ell_d\left(w_1^d,\mathsf{s}\right)}{\lambda_d}\right)\right),
  \end{align*}
  thus
  \begin{align*}
  	\left(Q^{(d)}_{W_1^{d-1}}\right)^{\frac{\lambda_d}{\lambda_{d-1}}}\propto \exp\left(-\frac{ L_{d-1}\left(w_1^{d-1},\mathsf{s}\right)}{\lambda_{d-1}}\right)\left(\sum_{w_d\in \W_d}\exp\left(-\frac{\ell_d\left(w_1^d,\mathsf{s}\right)}{\lambda_d}\right)\right)^{\frac{\lambda_d}{\lambda_{d-1}}}.
  \end{align*}
  Let $Z$ and $\bar{Z}$ denote the normalizing constants (partition functions) of $\left(Q^{(d)}_{W_1^{d-1}}\right)^{\frac{\lambda_d}{\lambda_{d-1}}}$ and $Q^{(d-1)}_{W_1^{d-1}}$, respectively.
  We have
\begin{align*}
  	&D\left(P_{W_1^{d-1}}\middle\|Q^{(d-1)}_{W_1^{d-1}}\right)-D\left(P_{W_1^{d-1}}\middle\|\left(Q^{(d)}_{W_1^{d-1}}\right)^{\frac{\lambda_2}{\lambda_1}}\right) \\
  	&= \sum_{w_1^{d-1}} \log\left( \frac{\left(Q^{(d)}_{W_1^{d-1}}\right)^{\frac{\lambda_d}{\lambda_{d-1}}}\left(w_1^{d-1}\right)}{Q^{(d-1)}_{W_1^{d-1}}\left(w_1^{d-1}\right)}\right) P_{W_1^{d-1}}\left(w_1^{d-1}\right)\\
  	&= \sum_{w_1^{d-1}} \log \left(\frac{\exp\left(-\frac{ L_{d-1}\left(w_1^{d-1},\mathsf{s}\right)}{\lambda_{d-1}}\right)\left(\sum_{w_d\in \W_d}\exp\left(-\frac{\ell_d\left(w_1^d,\mathsf{s}\right)}{\lambda_d}\right)\right)^{\frac{\lambda_d}{\lambda_{d-1}}}}{\exp\left(-\frac{L_{d-1}\left(w_1^{d-1},\mathsf{s}\right)}{\lambda_{d-1}} \right)}\times\frac{\bar{Z}}{Z}\right) P_{W_1^{d-1}}\left(w_1^{d-1}\right)\\
  	&=\frac{\lambda_d}{\lambda_{d-1}}\sum_{w_1^{d-1}} \log\left(\sum_{w_d}\exp \left(-\frac{\ell_d\left( w_1^d,\mathsf{s}\right)}{\lambda_d}\right)\right)P_{ W_1^{d-1}}\left(w_1^{d-1}\right)+\sum_{w_1^{d-1}} \log\left(\frac{\bar{Z}}{Z}\right)P_{W_1^{d-1}}\left(w_1^{d-1}\right)\\
  	&=\frac{\lambda_d}{\lambda_{d-1}}\sum_{w_1^{d-1}} \log\left(\sum_{w_d}\exp \left(-\frac{\ell_d\left( w_1^d,\mathsf{s}\right)}{\lambda_d}\right)\right)P_{ W_1^{d-1}}\left(w_1^{d-1}\right)+\log\left(\frac{\bar{Z}}{Z}\right)\\
  	&\cong\frac{\lambda_d}{\lambda_{d-1}}\sum_{w_1^{d-1}} \log\left(\sum_{w_d}\exp \left(-\frac{\ell_d\left( w_1^d,\mathsf{s}\right)}{\lambda_d}\right)\right)P_{ W_1^{d-1}}\left(w_1^{d-1}\right)\\
  	&= \frac{\lambda_d}{\lambda_{d-1}}\E\left[\log\left(\sum_{w_d}\exp \left(-\frac{\ell_d\left( W_1^{d-1}w_d,\mathsf{s}\right)}{\lambda_d}\right)\right)\right].
  \end{align*}
  Thus,
  \begin{align}
	&\lambda_{d-1}\left(D\left(P_{W_1^{d-1}}\middle\|Q^{(d-1)}_{W_1^{d-1}}\right)-D\left(P_{W_1^{d-1}}\middle\|\left(Q^{(d)}_{W_1^{d-1}}\right)^{\frac{\lambda_d}{\lambda_{d-1}}}\right) \right)\nonumber\\
	&\cong \lambda_d\E\left[\log\left(\sum_{w_d}\exp \left(-\frac{\ell_d\left( W_1^{d-1}w_d,\mathsf{s}\right)}{\lambda_d}\right)\right)\right]\label{eq: difference of entropy}
\end{align}
Based on \eqref{eq: cong2} and \eqref{eq: difference of entropy}, we deduce
\begin{align*}
	&\hat{\mathcal{L}}\left(P_{W_1^d} \right)+ \lambda_d\E\left[\log\left(\sum_{w_d}\exp \left(-\frac{\ell_d\left( W_1^{d-1}w_d,\mathsf{s}\right)}{\lambda_d}\right)\right)\right]\\
	&\cong \left(\sum_{k=1}^{d-2} (\lambda_k-\lambda_{k+1})D\left(P_{W_1^k}\middle\|U_{W_1^k}\right)\nonumber +\lambda_{d-1}D\left(P_{W_1^{d-1}}\middle\|\left(Q^{(d-1)}_{W_1^{d-1}}\right)\right)\right)\nonumber\\
	&\quad +\lambda_d D\left(P_{W_d|W_1^{d-1}}\middle\|Q^{(d)}_{W_d|W_1^{d-1}}\middle|P_{W_1^{d-1}}\right).
\end{align*}
Iterating this argument for $k=d-1,\dots,1$, we deduce that
\begin{align}
	&\E\left[{\ell}^{(\lambda)}(\mathsf{W},\mathsf{s})\right]+ \sum_{k=1}^{d} (\lambda_k-\lambda_{k+1})D\left(P_{W_1^k}\middle\|U_{W_1^k}\right)\nonumber\\
	&=\hat{\mathcal{L}}\left(P_{W_1^d} \right)+\sum_{k=1}^d \lambda_k\E\left[\log\left(\sum_{w_k}\exp \left(-\frac{\ell_d\left( W_1^{k-1}w_k,\mathsf{s}\right)}{\lambda_k}\right)\right)\right]\nonumber\\
	&\cong \sum_{k=1}^d \lambda_k D\left(P_{W_k|W_1^{k-1}}\middle\|Q^{(k)}_{W_k|W_1^{k-1}}\middle|P_{W_1^{k-1}}\right). \label{eq: vanishing entropies 1}
\end{align}
Note that, for all $1\leq k \leq d$,
\begin{align*}
	D\left(P_{W_1^k}\middle\|U_{W_1^k}\right)&=\log\left(|\W_1\times \cdots \times \W_k| \right)-H\left(W_1^k \right)\\
	&\cong -H\left(W_1^k \right).
\end{align*}
Thus, based on \eqref{eq: vanishing entropies 1}, we have
\begin{equation}
	\E\left[{\ell}^{(\lambda)}(\mathsf{W},\mathsf{s})\right]- \sum_{k=1}^{d} (\lambda_k-\lambda_{k+1})H\left(W_1^k \right)\cong \sum_{k=1}^d \lambda_k D\left(P_{W_k|W_1^{k-1}}\middle\|Q^{(k)}_{W_k|W_1^{k-1}}\middle|P_{W_1^{k-1}}\right). \nonumber
\end{equation}
Since, for all $1\leq k\leq d$, we have 
\begin{equation*}
	Q^{(k)}_{W_k|W_1^{k-1}}\left(w_k\middle|w_1^{k-1}\right)=P^{\star}_{W_k|W_1^{k-1}}\left(w_k\middle|w_1^{k-1}\right),
\end{equation*}
we can deduce that
\begin{equation}
	\E\left[{\ell}^{(\lambda)}(\mathsf{W},\mathsf{s})\right]- \sum_{k=1}^{d} (\lambda_k-\lambda_{k+1})H\left(W_1^k \right)\cong \sum_{k=1}^d \lambda_k D\left(P_{W_k|W_1^{k-1}}\middle\|P^{\star}_{W_k|W_1^{k-1}}\middle|P_{W_1^{k-1}}\right), \nonumber
\end{equation}
as desired.
\end{proof}
  Notice that for all $1\leq k \leq d$, $\overline{\ell}_k\left(w_1^{k-1},\s \right)$ is a Kolmogorov mean, the same type of which assumed in Lemma \ref{Kolmogorov mean lemma}.
  
Straightforwardly, the result extends for a random training sequence (vector) $\mathsf{S}\sim \mu^{\otimes n}$.
\begin{corollary}\label{multiscale entropy Corollary}
Assume that for all $1\leq k \leq d$, 
\begin{equation*}
	P^{*}_{W_k|W_1^{k-1}\mathsf{S}}\left(w_k\middle|w_1^{k-1}\s \right)=\frac{\exp\left(-\frac{\ell_k\left(w_1^k,\s \right)}{\lambda_k}\right)}{\sum_{w_k\in \W_k}\exp\left(-\frac{\ell_k\left(w_1^k,\s \right)}{\lambda_k}\right)}.
\end{equation*}
	Let
$
	P^{*}_{\mathsf{W}|\mathsf{S}}=P^{*}_{W_1|\mathsf{S}}P^{*}_{W_2|W_1\mathsf{S}}\dots P^{*}_{W_d|W_1^{d-1}\mathsf{S}}.
$
Then, 
\begin{equation*}
	P^{*}_{\mathsf{W}|\mathsf{S}}=\argmin_{P_{\mathsf{W}|\mathsf{S}}}\left\{\E \left[{\ell}^{(\lambda)}(\mathsf{W},\mathsf{S})\right] - \sum_{k=1}^d(\lambda_k-\lambda_{k+1})H\left(W_1^k\middle|\mathsf{S} \right) \right\},
\end{equation*}
where $(\mathsf{S},\mathsf{W})\sim P_{\mathsf{S}}P_{\mathsf{W}|\mathsf{S}}$.
\end{corollary}

Assume that our hypothesis set is realizable, that is, there exist parameters $\hat{\mathsf{w}}$ in our hypothesis index set such that $\psi_{k}(\cdot)=\mathcal{F}(\cdot,\hat{w}_k)$ (recall the definition of $\psi_{k}$ in Definition \ref{def: ladder decomposition}). Namely, we assume that function $f$ belongs to the hypothesis set.
Consider the following definition of risk:
\begin{definition}[Chained Risk]
	For any fixed $\mathsf{w}$, we define the chained risk as follows:
\begin{equation*}
	L_{\mu}^{\mathfrak{C}}(\mathsf{w}):= \E \left[\sum_{k=1}^d \left(\ell_k\left(w_1^k,{\mathsf{S}}\right) -\ell_k\left(w_1^{k-1}\hat{w}_k,{\mathsf{S}}\right)\right)\right].
\end{equation*}
\end{definition}
The training mechanism of the learning model chooses the values of parameters $w_1,...,w_d$ sequentially. At the $k$th stage, choosing $w_k$ instead of the true target function parameter $\hat{w}_k$ results in the following difference between the statistical risks: 
\begin{equation*}
	\E\left[\ell_k\left(w_1^k,{\mathsf{S}}\right)\right] -\E \left[\ell_k\left(w_1^{k-1}\hat{w}_k,{\mathsf{S}}\right)\right].
\end{equation*} 
The chained risk is equal to the accumulation of these deviations of statistical risk at each of the $d$ stages of the training mechanism. Clearly, we have $L_{\mu}^{\mathfrak{C}}(\hat{\mathsf{w}})=0$. Intuitively and roughly speaking, if the chained risk of $\mathsf{w}$ is small, then $h_{\mathsf{w}}$ should be close to the target function $h_{\hat{\mathsf{w}}}=f$. In the next theorem, we derive an upper bound on the expected value of the chained risk $\E \left[L_{\mu}^{\mathfrak{C}}(\mathsf{W})\right]$ when the model is trained by the multiscale entropic mechanism (sampling $\mathsf{w}$ from $P^{\star}_{\mathsf{W}|\mathsf{S}}$). Then, in the next subsection, we show a condition on the data instance distribution $\mu$ with which small chained risk implies small statistical risk.
\begin{theorem}\label{thm: chained risk of multiscale entropic training} Let $(\mathsf{S},\mathsf{W})\sim P_{\mathsf{S}}P^{\star}_{\mathsf{W}|\mathsf{S}}$. The average chained risk of the training mechanism $P^{\star}_{W|S}$ satisfies the following inequality:
	\begin{equation*}
		\E \left[L_{\mu}^{\mathfrak{C}}(\mathsf{W})\right]\leq \sum_{k=1}^d \left(2(\lambda_k-\lambda_{k+1})\left(\sum_{m=1}^k \log |\W_m|\right)+\frac{\rho_k^2}{2n(\lambda_k-\lambda_{k+1})}\right).
	\end{equation*}
\end{theorem}
\begin{proof}
Let $\bar{\mathsf{W}}$ and $\bar{\mathsf{S}}=\left(\bar{X}_1,\dots, \bar{X}_n\right)$ be independent copies of $\mathsf{W}$ and $\mathsf{S}$, respectively, and assume that $\bar{\mathsf{W}}$ and $\bar{\mathsf{S}}$ are independent from each other. Namely, $\left(\bar{\mathsf{W}},\bar{\mathsf{S}}\right)\sim P_\mathsf{W}P_\mathsf{S}$. Clearly,
\begin{equation*}
\E\left[{\ell}^{\mathfrak{C}}(\bar{\mathsf{W}},\bar{\mathsf{S}})\right]=\E \left[L_{\mu}^{\mathfrak{C}}(\mathsf{W})\right].  
\end{equation*}
Recall that for any fixed $\mathsf{w}=(w_1,\dots,w_d)$ and any $x$ and all $1\leq k \leq d$, $h_k(x)$ is defined as in \eqref{model formulation 1}. Based on \eqref{eq: definition of k-risk}, we have
\begin{equation*}
	\ell_k\left({w}_1^k,x\right)=
	\begin{cases}
		\left|\gamma_k\left(h_{k-1}\left(\frac{x}{\gamma_k}\right)+\mathcal{F}\left(h_{k-1}\left(\frac{x}{\gamma_k}\right),w_k\right) \right)-f(x)\right| &\textrm{ if } x\in\X_k \\
		0 &\textrm{ if } x\notin\X_k,
	\end{cases}
\end{equation*}
and 
\begin{equation*}
	\ell_k\left({w}_1^{k-1}\hat{w}_k,x\right)=
	\begin{cases}
		\left|\gamma_k\left(h_{k-1}\left(\frac{x}{\gamma_k}\right)+\mathcal{F}\left(h_{k-1}\left(\frac{x}{\gamma_k}\right),\hat{w}_k\right) \right)-f(x)\right| &\textrm{ if } x\in\X_k \\
		0 &\textrm{ if } x\notin\X_k.
	\end{cases}
\end{equation*}
Since for any $a,b\in \R$, it is easily seen that $||a|-|b||\leq |a-b|$, we deduce that
 \begin{align*}
	\left|\ell_k\left({w}_1^k,x\right) -\ell_k\left({w}_1^{k-1}\hat{w}_k,x\right)\right|&\leq \gamma_k\left|\mathcal{F}\left(h_{k-1}\left(\frac{x}{\gamma_k}\right),{w}_k\right)-\mathcal{F}\left(h_{k-1}\left(\frac{x}{\gamma_k}\right),\hat{w}_k\right) \right|\\
	&\leq \gamma_k\left(\left|\mathcal{F}\left(h_{k-1}\left(\frac{x}{\gamma_k}\right),{w}_k\right)\right|+\left|\mathcal{F}\left(h_{k-1}\left(\frac{x}{\gamma_k}\right),\hat{w}_k\right)\right| \right)\\
	&\leq 2\gamma_k\rho_k, \numberthis \label{eq: based on bounded regularization}
\end{align*} 
where \eqref{eq: based on bounded regularization} is obtained by using \eqref{model formulation 1}. Hence, based on Azuma--Hoeffding's inequality (Lemma \ref{lem: Azuma Hoeffding}), for any fixed $\mathsf{w}$ and all $1\leq k \leq d$, $\ell_k\left({w}_1^k,{\mathsf{S}}\right) -\ell_k\left({w}_1^{k-1}\hat{w}_k,{\mathsf{S}}\right)$ is $4\gamma_k\rho_k/\sqrt{n}$-subgaussian. Thus, $\ell_k\left(\bar{W}_1^k,\bar{\mathsf{S}}\right) -\ell_k\left(\bar{W}_1^{k-1}\hat{w}_k,\bar{\mathsf{S}}\right)$ is $4\gamma_k\rho_k/\sqrt{n}$-subgaussian as well. Based on Lemma \ref{lemma: subgaussian difference independence}, we can write
\begin{align*}
	&\E \left[{\ell}^{\mathfrak{C}}(\bar{\mathsf{W}},\bar{\mathsf{S}})\right]-\E\left[{\ell}^{\mathfrak{C}}(\mathsf{W},\mathsf{S})\right]\\
	&= \sum_{k=1}^d\left(\E  \left[ \ell_k\left(\bar{W}_1^k,\bar{\mathsf{S}}\right) -\ell_k\left(\bar{W}_1^{k-1}\hat{w}_k,\bar{\mathsf{S}}\right)\right]-\E \left[ \ell_k\left(W_1^k,{\mathsf{S}}\right) -\ell_k\left(W_1^{k-1}\hat{w}_k,{\mathsf{S}}\right)\right]\right)\\
	&\leq \sum_{k=1}^d \left(\bl_k\left(\sum_{m=1}^k \log |\W_m|-H\left(W_1^k\middle|\mathsf{S}  \right)\right)+\frac{8\gamma_k^2\rho_k^2}{n\bl_k}\right), \numberthis \label{Chaining bound}
\end{align*}
where we define for all $1\leq k \leq d$,
\begin{equation*}
	\bl_k:=\lambda_k-\lambda_{k+1}.
\end{equation*}
Therefore,
\begin{align}
	\E \left[L_{\mu}^{\mathfrak{C}}(\mathsf{W})\right]&=   \E\left[{\ell}^{\mathfrak{C}}(\bar{\mathsf{W}},\bar{\mathsf{S}})\right]\nonumber\\
	&\leq \E\left[{\ell}^{\mathfrak{C}}(\mathsf{W},\mathsf{S})\right]+\sum_{k=1}^d \left(\bl_k\left(\sum_{m=1}^k \log |\W_m|-H\left(W_1^k\middle|\mathsf{S}  \right)\right)+\frac{8\gamma_k^2\rho_k^2}{n\bl_k}\right)\label{second}\\
	&\leq \E\left[{\ell}^{(\lambda)}(\mathsf{W},\mathsf{S})\right] +\sum_{k=1}^d \left(\bl_k\left(\sum_{m=1}^k \log |\W_m|-H\left(W_1^k\middle|\mathsf{S}  \right)\right)+\frac{8\gamma_k^2\rho_k^2}{n\bl_k}\right)\label{second2}\\
	&\leq \E\left[{\ell}^{(\lambda)}(\hat{\mathsf{w}},{\mathsf{S}})\right]+\sum_{k=1}^d \left(\bl_k\left(\sum_{m=1}^k \log |\W_m|\right)+\frac{8\gamma_k^2\rho_k^2}{n\bl_k}\right)\label{third}\\
	&\leq  \sum_{k=1}^d \left(\bl_k\left(\sum_{m=1}^k \log |\W_m|\right)+\lambda_k\log |\W_k|+\frac{8\gamma_k^2\rho_k^2}{n\bl_k}\right)\label{fourth}\\
	&= \sum_{k=1}^d \left(2\bl_k\left(\sum_{m=1}^k \log |\W_m|\right)+\frac{8\gamma_k^2\rho_k^2}{n\bl_k}\right)\label{fifth},
\end{align}
where \eqref{second} is obtained by rewriting \eqref{Chaining bound}, \eqref{second2} is based on Lemma \ref{Kolmogorov mean lemma}, \eqref{third} is obtained based on Corollary \ref{multiscale entropy Corollary} and by replacing $P^*_{\mathsf{W}|\mathsf{S}}$ with the conditional distribution $P_{\mathsf{W}|\mathsf{S}} = \delta_{\hat{\mathsf{w}}}$ (the Dirac measure on $\hat{\mathsf{w}}$), and \eqref{fourth} is again based on Lemma \ref{Kolmogorov mean lemma} and by noting that $\ell_k\left(\hat{w}_1^k,x\right)=0$ for all $1\leq k \leq d$ and all $x\in\X$.
\end{proof}
Optimizing the bound in \eqref{fifth} over the values of $(\lambda_1,\dots,\lambda_d)$ gives the following result:
\begin{corollary}\label{cor: optimized bound}
	Assume that $(\lambda_1,\dots,\lambda_d)$ are chosen such that for all $1\leq k \leq d$,
\begin{equation*}
	\bl_k =\lambda_k-\lambda_{k+1}= \frac{2\gamma_k\rho_k}{\sqrt{n\left(\sum_{m=1}^k \log |\W_m|\right)}}\cdot
\end{equation*}
Then, the right side of \eqref{fifth} is minimized with respect to $(\lambda_1,\dots,\lambda_d)$. In this case, the bound simplifies to the following form:
\begin{align}\label{ineq: bound on chained risk}
	\E \left[L_{\mu}^{\mathfrak{C}}(\mathsf{W})\right]&\leq \frac{4}{\sqrt{n}}\sum_{k=1}^d \gamma_k\rho_k\sqrt{\sum_{m=1}^k \log |\W_m|}.
\end{align}
\end{corollary}
\subsection{Bounding Statistical Risk Based on Chained Risk: An Example}\label{sec: bounding chained risk example}
The analysis, up to now, did not require any restrictions on the data instance distribution $\mu$.
In this subsection, we give an example of a condition on $\mu$ for which small chained risk implies small statistical risk. 

Assume that the instance distribution $\mu$ defined on $\X$ has the following power-law probability density function with shape parameter $\alpha\geq 1$: 
\begin{equation*}
	q(x)=\frac{1}{C'|x|^{\alpha}} \textrm{ for all } x\in \X,
\end{equation*}
where $C'=\int_{\X}|x|^{-\alpha}\dd x$. The density function $q(x)$ is scale-invariant: For all $x\in \X$, we have 
\begin{equation*}
	q\left(\frac{x}{\beta}\right)=\beta^{\alpha}q(x).
\end{equation*}
Given this assumption, in the following result we show that the chained risk can bound the statistical risk from above:

\begin{theorem}
	If $\mu$ has a power-law probability density function with shape parameter $\alpha$, then for any $\mathsf{w}$, we have 
	\begin{equation*}
		\left(1-\beta^{1-\alpha}\left(1+C_1 R(1-\beta^{-1})\right)\right)L_{\mu}(\mathsf{w})\leq L^{\mathfrak{C}}_{\mu}(\mathsf{w}).
	\end{equation*}
\end{theorem}
\begin{proof} For any $1\leq k \leq d$, let $\hat{h}_k(x)$ denote the $k$th level of the model given weight parameters $\hat{\mathsf{w}}$ and input $x$. Assume that $\psi_{\hat{w}_k}(\cdot):=\mathcal{F}(\cdot, \hat{w}_k)$. For any $x\in \X_k$, we have
	\begin{align*}
		\ell_k\left(w_1^{k-1}\hat{w}_{k},x\right)&=\left|\gamma_k(I+\psi_{\hat{w}_k})\circ h_{k-1}\left(\frac{x}{\gamma_k}\right)-f(x)\right|\\
		&=\left|\gamma_k(I+\psi_{\hat{w}_k})\circ h_{k-1}\left(\frac{x}{\gamma_k}\right)-\gamma_k(I+\psi_{\hat{w}_k})\circ \hat{h}_{k-1}\left(\frac{x}{\gamma_k}\right)\right|\\
		&=\gamma_k\left|h_{k-1}\left(\frac{x}{\gamma_k}\right)-\hat{h}_{k-1}\left(\frac{x}{\gamma_k}\right) + \psi_{\hat{w}_k}\left(h_{k-1}\left(\frac{x}{\gamma_k}\right)\right)-\psi_{\hat{w}_k}\left(\hat{h}_{k-1}\left(\frac{x}{\gamma_k}\right)\right)\right|\\
		&\leq \gamma_k \left|h_{k-1}\left(\frac{x}{\gamma_k}\right)-\hat{h}_{k-1}\left(\frac{x}{\gamma_k}\right)\right|+\gamma_k \left|\psi_{\hat{w}_k}\left(h_{k-1}\left(\frac{x}{\gamma_k}\right)\right)-\psi_{\hat{w}_k}\left(\hat{h}_{k-1}\left(\frac{x}{\gamma_k}\right)\right)\right|\\
		&\leq \gamma_k \left|h_{k-1}\left(\frac{x}{\gamma_k}\right)-\hat{h}_{k-1}\left(\frac{x}{\gamma_k}\right)\right| \\
		&\quad +C_1 \varepsilon(\beta-1) \beta^{k-1}\gamma_k \left|h_{k-1}\left(\frac{x}{\gamma_k}\right)-\hat{h}_{k-1}\left(\frac{x}{\gamma_k}\right)\right|\numberthis \label{lipschitz psi}\\
		&=\gamma_k\left(1+C_1 \varepsilon(\beta-1) \beta^{k-1}\right)\left|h_{k-1}\left(\frac{x}{\gamma_k}\right)-\hat{h}_{k-1}\left(\frac{x}{\gamma_k}\right)\right|,
	\end{align*}
	where in \eqref{lipschitz psi}, we used the fact that $\psi_{\hat{w}_k}$ is $C_1 \varepsilon(\beta-1) \beta^{k-1}\gamma_k$-Lipschitz (based on Theorem \ref{Ladder Decomposition Theorem}).
	Now, define 
	\begin{equation*}
		x':=\frac{\gamma_{k-1}}{\gamma_k}x=\frac{x}{\beta}.
	\end{equation*}
	We can observe that $x'\in \X_{k-1}$ and the transformation $x\to x'$ is a bijection between $\X_k$ and $\X_{k-1}$. Therefore
	\begin{align*}
		\ell_k\left(w_1^{k-1}\hat{w}_{k},x\right)&\leq \gamma_k\left(1+C_1 \varepsilon(\beta-1) \beta^{k-1}\right)\left|h_{k-1}\left(\frac{x}{\gamma_k}\right)-\hat{h}_{k-1}\left(\frac{x}{\gamma_k}\right)\right|\\
		&=\gamma_k\left(1+C_1 \varepsilon(\beta-1) \beta^{k-1}\right)\left|h_{k-1}\left(\frac{x'}{\gamma_{k-1}}\right)-\hat{h}_{k-1}\left(\frac{x'}{\gamma_{k-1}}\right)\right|\\
		&=\frac{\gamma_k}{\gamma_{k-1}}\left(1+C_1 \varepsilon(\beta-1) \beta^{k-1}\right)\left|\gamma_{k-1}h_{k-1}\left(\frac{x'}{\gamma_{k-1}}\right)-\gamma_{k-1}\hat{h}_{k-1}\left(\frac{x'}{\gamma_{k-1}}\right)\right|\\
		&=\beta\left(1+C_1 \varepsilon(\beta-1) \beta^{k-1}\right)\ell_{k-1}\left(w_1^{k-1},x'\right). \numberthis \label{ineq: relating loss k to loss k-1}
	\end{align*} 
	Let ${X}$ be distributed according to density $q(x)$. Based on \eqref{ineq: relating loss k to loss k-1}, we have
	\begin{align*}
		\E\left[\ell_k\left(w_1^{k-1}\hat{w}_{k},{X}\right)\right]&=\int_{x\in \X_k}\ell_k\left(w_1^{k-1}\hat{w}_{k},x\right)q(x)\dd x \\
		&\leq \int_{x\in \X_k}\beta\left(1+C_1 \varepsilon(\beta-1) \beta^{k-1}\right)\ell_{k-1}\left(w_1^{k-1},x'\right)q(x)\dd x\\
		&= \int_{x\in \X_k}\beta\left(1+C_1 \varepsilon(\beta-1) \beta^{k-1}\right)
		\ell_{k-1}\left(w_1^{k-1},x'\right)\beta^{-\alpha}q\left(x'\right)\dd x\\
		&= \int_{x'\in \X_{k-1}}\beta^{1-\alpha}\left(1+C_1 \varepsilon(\beta-1) \beta^{k-1}\right)\ell_{k-1}\left(w_1^{k-1},x'\right)q\left(x'\right)\dd x'\\
		&= \beta^{1-\alpha}\left(1+C_1 \varepsilon(\beta-1) \beta^{k-1}\right)\E\left[\ell_{k-1}\left(w_1^{k-1},{X}\right)\right]\\
		&\leq \beta^{1-\alpha}\left(1+C_1 R(1-\beta^{-1}) \right)\E\left[\ell_{k-1}\left(w_1^{k-1},{X}\right)\right]. \numberthis \label{ineq: relating 1}
	\end{align*}
Thus, 
\begin{align*}
	L^{\mathfrak{C}}_{\mu}(\mathsf{w})&= \sum_{k=1}^d\left(\E\left[\ell_k\left(w_1^k,{X}\right)\right] -\E\left[\ell_k\left(w_1^{k-1}\hat{w}_{k},{X}\right)\right] \right)\\
	&\geq \sum_{k=1}^d\left(\E\left[\ell_k\left(w_1^k,{X}\right)\right] -\beta^{1-\alpha}\left(1+C_1 R(1-\beta^{-1})\right)\E\left[\ell_{k-1}\left(w_1^{k-1},{X}\right)\right] \right)\numberthis \label{ineq: relating 2}\\
	&=\E \left[\sum_{k=1}^d \ell_k\left(w_1^k,{X}\right) \right] - \beta^{1-\alpha}\left(1+C_1 R(1-\beta^{-1})\right)\E\left[\sum_{k=1}^d \ell_{k-1}\left(w_1^{k-1},{X}\right)\right]\\
	&\geq \E \left[\sum_{k=1}^d \ell_k\left(w_1^k,{X}\right) \right] - \beta^{1-\alpha}\left(1+C_1 R(1-\beta^{-1})\right)\E\left[\sum_{k=1}^d \ell_k\left(w_1^{k},{X}\right)\right]\\
	&=\left(1-\beta^{1-\alpha}\left(1+C_1 R(1-\beta^{-1})\right)\right)L_{\mu}(\mathsf{w}),
\end{align*}
where \eqref{ineq: relating 2} is based on \eqref{ineq: relating 1}. 
\end{proof}
Therefore, based on Corollary \ref{cor: optimized bound}, given sufficiently large $\alpha>1$ for which 
$$\left(1-\beta^{1-\alpha}\left(1+C_1 R(1-\beta^{-1})\right)\right)>0,$$ the following inequality holds:
\begin{equation}\label{ineq: risk bound}
	\E\left[L_{\mu}(\mathsf{W}) \right]\leq \frac{4}{\sqrt{n}}\left(\frac{\sum_{k=1}^d \gamma_k\rho_k\sqrt{\sum_{m=1}^k \log |\W_m|}}{\left(1-\beta^{1-\alpha}\left(1+C_1 R(1-\beta^{-1})\right)\right)}\right).
\end{equation}
Given the realizability assumption on the hypothesis set, the regular union bound applied to the empirical-risk-minimizing hypothesis yields
\begin{equation}\label{ineq: ERM risk bound}
	\E\left[L_{\mu}(\mathsf{W}_{\textrm{ERM}}) \right]\leq \frac{\left(\sum_{k=1}^d \rho_k\right)}{\sqrt{n}}\sqrt{\sum_{m=1}^d \log |\W_m|}.
\end{equation}
Ignoring the effect of $\left(1-\beta^{1-\alpha}\left(1+C_1 R(1-\beta^{-1})\right)\right)$, the following example shows that the right side of \eqref{ineq: risk bound} can be quite smaller than the right side of \eqref{ineq: ERM risk bound}:
\begin{example}
	Let $R/\epsilon := \bar{R}$ and $\beta=(\bar{R})^{1/d}$. Recall that we have $\gamma_k=\beta^{k-d}$ as in \eqref{eq: scale geometric sequence}. Assume that $\rho_k=\rho_0\beta^k$ for all $k=1,\dots, d$ and $|\W_1|=\cdots =|\W_d|$. We compute the following ratio
	\begin{align*}
		\Lambda = \left(\frac{\sum_{k=1}^d \gamma_k\rho_k\sqrt{\sum_{m=1}^k \log |\W_m|}}{\left(\sum_{k=1}^d \rho_k\right)\sqrt{\sum_{m=1}^d \log |\W_m|}}\right)^2&=\left(\frac{\sum_{k=1}^d \beta^{2k-d}\sqrt{k}}{\sum_{k=1}^d \beta^k\sqrt{d}}\right)^2.
	\end{align*}
	The power of two exists to compare the bounds on the required number of samples $n$.
	For example, given $\bar{R}=10$ and $d=20$, we obtain $\Lambda \approx 0.2648$.
\end{example}
\section{A Parameterization Example and Analysis of its Representation Power}\label{sec: bounded norm}
In this section, we show with an example that a set of Lipschitz functions, i.e. bounded Lipschitz norm, can be represented with a parameterized model with bounded-norm parameters.
Note that the range of $f_{[\gamma_k]}$, which is the domain of $\psi_k$, is an interval subset of $(-M_1R,M_1R)$ that includes $0$, where $R=\varepsilon\beta^d$.

Let $\Psi(x)$ be a $\phi_1$-Lipschitz and $\phi_2$-smooth function with support $D_{\Psi}=(a_1,a_2)\subseteq (-M_1R,M_1R)$ such that $0\in D_{\Psi}$ and $\Psi(0)=0$. Later in this section, we will replace $\Psi$ with each $\psi_k$ of the ladder decomposition of the diffeomorphism $f$ in \eqref{ladder decomposition equation}.
For any $x\in D_{\Psi}$, we can write  
\begin{align}
	\Psi(x)&=\int_{a_1}^{x} \Psi'(b)\dd b+ \Psi(a_1)\nonumber\\
				 &=\int_{a_1}^{a_2} \Psi'(b)H(x-b)\dd b+\Psi(a_1),\label{step function rep}
\end{align}
where $H(x)$ is the Heaviside (unit) step function.
Note that due to $\Psi$ being $\phi_1$-Lipschitz, for all $b\in D_{\Psi}$, we have $|\Psi'(b)|\leq \phi_1$. Moreover, since $\Psi(0)=0$, we conclude that $|\Psi(a_1)|\leq \phi_1|a_1|\leq \phi_1M_1R$. 

For a given integer $\tau\geq 2$, let a $\tau$-width two-layer network with \emph{continuous} parameters $w:=\{w_1,\dots,w_{\tau},w^{(c)}\}$ be defined for any $x\in(-M_1R,M_1R)$ as
\begin{equation}\label{finite width network def}
	\bar{\psi}_w^{({\tau})}(x):= \sum_{j=1}^{\tau} w^{(j)}H(x-b_j)+w^{(c)},
\end{equation}
where for all $1\leq j\leq \tau$,
$
	b_j:= (-1+2j/\tau)M_1R,
$
\begin{equation*}
\begin{cases}
	w^{(j)}\gets \frac{2M_1R\Psi'(b_j)}{{\tau}} &\textrm{ if } b_j\in D_{\Psi}\\
	w^{(j)}\gets 0 &\textrm{ if } b_j\notin D_{\Psi},
\end{cases}	
\end{equation*}
and
\begin{equation*}
	w^{(c)}\gets \Psi(a_1).
\end{equation*}
We can view $\bar{\psi}_w^{(\tau)}(x)$ in \eqref{finite width network def} as a Reimann sum approximation to the integral representation of $\Psi(x)$ in \eqref{step function rep}. 
Using Lemma \ref{Reimann sum approximation lemma}, we deduce the following result:
\begin{lemma}\label{Reimann sum approximation lemma network}
	For all $x\in D_{\Psi}$, we have
	\begin{equation*}
		\left|\int_{a_1}^{a_2} \Psi'(b)H(x-b)\dd b - \sum_{j=1}^{\tau} w^{(j)} H(x-b_j)\right|\leq \frac{2M_1R\phi_2}{\tau}.
	\end{equation*}
\end{lemma}
\begin{proof}
	Let $j_1$ be the smallest integer $j$ such that $b_j\in D_{\Psi}$, and let $j_2$ be the largest integer $j$ such that $ (-1+{2j}/{\tau})M_1R\leq x$.
	We have
	\begin{align*}
		\left|\int_{a_1}^{a_2} \Psi'(b)H(x-b)\dd b - \sum_{j=1}^{\tau} w^{(j)} H(x-b_j)\right|&=\left|\int_{a_1}^{x} \Psi'(b)\dd b - \sum_{j=j_1}^{j_2} \frac{2M_1R\Psi'(b_j)}{{\tau}}\right|\\
		&\leq \frac{\phi_2(x-a_1)^2}{2(j_2-j_1)}\\
		&\leq \frac{\phi_2(x-a_1)^2}{2\frac{(x-a_1)}{2M_1R}\tau}\\
		&\leq \frac{\phi_2(x-a_1)M_1R}{\tau}\\
		&\leq \frac{2M_1^2R^2\phi_2}{\tau}.
	\end{align*}
\end{proof}

Now, we proceed to discretize the weights of the network. Let $\eta>0$ be the precision level of the weights. We discretize the weights of $\bar{\psi}_w^{(\tau)}$ in \eqref{finite width network def} by choosing the closest real number in $\eta \mathbb{Z}$ to each weight.
For any function $\Psi$, let its approximate $\tau$-width two-layer network with \emph{discretized} parameters at discretization $\eta$
 be defined as
\begin{equation}\label{finite width discretized network def}
	\psi_w^{({\tau,\eta})}(x):= \sum_{j=1}^{\tau} \left[w^{(j)}\right]_{\eta}H(x-b_j)+\left[w^{(c)}\right]_{\eta}.
\end{equation}
We have the following bound on the approximation error of the finite-width two-layer network:
\begin{proposition}[Approximation Error]
	For all $x\in D_{\Psi}$,
	\begin{equation*}
	\left|\psi_{w}^{({\tau},\eta)}(x)-\Psi(x)\right|\leq \frac{(\tau+1) \eta}{2}+\left(\frac{2M_1^2R^2}{{\tau}}\right)\phi_2.
\end{equation*}
\end{proposition}
\begin{proof} 	
Based on Lemma \ref{Reimann sum approximation lemma network}, for all $x\in D_{\Psi}$ we have 
\begin{equation*}
	\left|\bar{\psi}_{w}^{({\tau})}(x) - \Psi(x)\right|\leq \frac{2M_1R\phi_2}{\tau}.
\end{equation*}
		Thus, for all $x\in D_{\Psi}$, we can deduce 
	\begin{align*}
		\left|\psi_{w}^{({\tau, \eta})}(x)-\Psi(x)\right|& = \left|\psi_{w}^{({\tau, \eta})}(x)- \bar{\psi}_{w}^{({\tau})}(x) + \bar{\psi}_{w}^{({\tau})}(x) - \Psi(x)\right|\\
		&\leq \left|\psi_{w}^{({\tau, \eta})}(x)- \bar{\psi}_{w}^{({\tau})}(x)\right| + \left|\bar{\psi}_{w}^{({\tau})}(x) - \Psi(x)\right|\\
		&\leq \sum_{j=1}^{\tau} \left|w^{(j)}-\left[w^{(j)}\right]_{\eta} \right| + \left|w^{(c)} - \left[w^{(c)}\right]_{\eta} \right|+\frac{2M_1^2R^2\phi_2}{\tau}\\
		&\leq \frac{\tau \eta}{2}+ \frac{\eta}{2}+ \frac{2M_1^2R^2\phi_2}{\tau}\\
		&\leq  \frac{(\tau+1) \eta}{2}+ \frac{2M_1^2R^2\phi_2}{\tau}. 
	\end{align*}	
\end{proof}
In the following proposition, we show an upper bound on the $\ell_1$-norm of the finite-width discretized network $\psi^{(\tau,\eta)}_{w}$ derived from function $\Psi(x)$ in \eqref{finite width discretized network def}.

\begin{proposition}[Bounded Norm]
	The $\ell_1$-norm of the network $\psi_w^{(\tau,\eta)}$, defined as the sum of absolute values of its weights $w$  
	satisfies
	\begin{equation*}
		|w|_1\leq 3M_1R\phi_1+\left(\frac{4M_1^2R^2}{\tau}\right)\phi_2+\frac{(\tau+1) \eta}{2}.
	\end{equation*}
\end{proposition}
\begin{proof}
	$\Psi(x)$ is $\phi_2$-smooth, therefore $\Psi'(x)$ is $\phi_2$-Lipschitz. Thus, $|\Psi'(x)|$ is $\phi_2$-Lipschitz as well. Recall that $|\Psi'(b)|\leq \phi_1M_1R$ for all $b\in D_{\Psi}$.
We have
\begin{align}
	|w|_1 &=\sum_{j=j_1}^{j_2}\left|\left[w^{(j)} \right]_{\eta}\right|+\left|\left[w^{(c)} \right]_{\eta}\right|\nonumber\\
				&\leq \sum_{j=j_1}^{j_2}\left|\frac{2M_1R\Psi'(b_j)}{{\tau}}\right|+|\Psi(a_1)|+\frac{(\tau+1) \eta}{2}\nonumber\\
				&\leq \int_{a_1}^{a_2}|\Psi'(b)|\dd b +\frac{4\phi_2M_1^2R^2}{{\tau}}+{\phi_1M_1R}+\frac{(\tau+1) \eta}{2}\nonumber\\
				&\leq 2M_1R\phi_1+\frac{4\phi_2M_1^2R^2}{{\tau}}+\phi_1M_1R+\frac{(\tau+1) \eta}{2}\nonumber\\
				&=3M_1R\phi_1+\left(\frac{4M_1^2R^2}{\tau}\right)\phi_2+\frac{(\tau+1) \eta}{2}.\label{L_1 norm finite width}
\end{align}

\end{proof}
We now replace function $\Psi$ in the previous arguments with $\psi_k$ for any $1\leq k \leq d$. Based on Theorem \ref{Ladder Decomposition Theorem}, we take $\phi_1\gets C_1 R\beta^{k-d-1}(\beta-1)$ and $\phi_2\gets C_2$. Define
\begin{equation*}
	\rho_k:= 3M_1C_1 R^2\beta^{k-d-1}(\beta-1)+\left(\frac{4M_1^2R^2}{\tau}\right)C_2+\frac{(\tau+1) \eta}{2}.
\end{equation*}
 Suppose $\W_k$, the set of weights for our learning model at level $k$, is 
\begin{equation*}
	\W_k:=\left\{w=\left(w^{(1)},\dots,w^{(\tau_k)},w^{(c)}\right): \forall 1\leq j \leq \tau \hspace{2mm} w^{(j)}\in  \eta \mathbb{Z}, w^{(c)}\in \eta\mathbb{Z}, |w|_1\leq \rho_k \right\}.
\end{equation*}
	For all $1\leq k \leq d$, in the recursive definition of the model \eqref{model formulation 1}, we define
	
	\begin{equation}\nonumber
	\mathcal{F}\left(h_{k-1}(x),w_k \right):=\psi_{w_k}^{(\tau,\eta)}\left(h_{k-1}(x)\right).
\end{equation}
Therefore, 
\begin{equation}\nonumber
	h_k(x)= \left(\psi_{w_k}^{({\tau_k,\eta_k})}+I\right)\circ \dots \circ \left(\psi_{w_{2}}^{({\tau_2,\eta_2})}+I\right)\circ\left(\psi_{w_1}^{({\tau_1,\eta_1})}+I\right)(x).
\end{equation}
Such bounded regularization on the weights of the learning model immediately implies the following property:
\begin{proposition}[Bounded Output]
	Let $w\in \W_k$. Then, for all $x\in (-M_1R,M_1R)$, we have 
\begin{equation*}
	\left|\psi_{w}^{\tau,\eta}(x)\right|\leq \rho_k.
\end{equation*}
\end{proposition}
\begin{proof}
	We can write
	\begin{align*}
		\left|\psi_{w}^{\tau,\eta}(x)\right|&= \left|\sum_{j=1}^{\tau} w^{(j)}H\left(x-b_j\right) \right|\\
		&\leq \sum_{j=1}^{\tau}\left|w^{(j)} \right|\\
		&\leq\rho_k.
	\end{align*}
\end{proof}

\section{Conclusions}\label{sec: Conclusions}
In this paper, we presented an entropy-based learning model to exploit the multiscale structure of data domains and the smoothness of target functions. We first showed the definition of ladder decompositions for diffeomorphisms and studied Lipschitz continuity and smoothness of the levels of this decomposition. Then, we proved that the self-similar maximum-entropy type training achieves low {chained risk}. We showed that if the data distribution $\mu$ is a power-law distribution, then the chained risk can bound the statistical risk from above.  Hence, this yields that the multiscale-entropic training mechanism achieves low statistical risk. Finally, we provided an example of a parameterized model with bounded-norm parameters.

Our proposed learning model has the following merits: It is a hierarchical learning model with interpretable levels. The training is carried out to make the mapping between the input and any level of the learning model approximate a dilation of the target function. This makes the role of each level of the learning model have an interpretation, in contrast to black-box hierarchical models. Another merit of the proposed model is the computational point of view.  The amount of computation required on computing the output of the learned model given a particular data instance as input is proportionate to the complexity of that instance. 
		Since, by assumption, data instances are distributed heterogeneously at different scales and complexities, 
	this fact can result in significant computational savings and higher inference speed.
	On the other hand, the proposed learning model can provide computational savings when several different users each require learning the target function at a particular scale of data instances. Moreover, training of current machine learning models on massive datasets may take a very long time. We showed that as our training mechanism consists of different stages if for any reason this mechanism terminates after a stage, one can still guarantee a useful model with which it can accurately predict the label of data instances with norms smaller than a particular value depending on the stage. Finally,  as the statistical analysis of the risk of the trained compositional model is tailored to the hierarchical training mechanism and takes its multiscale structure into account in deriving the bound on its statistical risk, the bound can be much sharper than a uniform convergence bound for the empirical-risk-minimizing hypothesis.

\vskip 0.2in 
\bibliographystyle{unsrt}
\bibliography{Biblio_Dec22.bib}
\end{document}